\documentclass[letterpaper, 10 pt, conference]{cls/ieeeconf}
\IEEEoverridecommandlockouts
\usepackage[table,usenames,dvipsnames]{xcolor}      % color
\usepackage[noadjust]{cite}
\usepackage{amsmath,amssymb,amsfonts,amsthm,dsfont,mathtools}

\usepackage{graphicx,tabularx,adjustbox}
\usepackage{multirow}
\usepackage[font=small]{caption}
\usepackage[font=footnotesize]{subcaption}

\usepackage{algorithmic}
\usepackage[ruled,vlined]{algorithm2e}

\usepackage{dirtytalk}
\allowdisplaybreaks
\usepackage[breaklinks=true, colorlinks, bookmarks=true, citecolor=Black, urlcolor=Violet,linkcolor=Black]{hyperref}

\newcommand{\prl}[1]{\left(#1\right)}

\newcommand{\crl}[1]{\left\{#1\right\}}
\newcommand{\norm}[1]{\left\lVert#1\right\rVert}
\newcommand{\Int}{\operatorname{Int}}

\newcommand{\scaleMathLine}[2][1]{\resizebox{#1\linewidth}{!}{$\displaystyle{#2}$}}

\newtheorem{proposition}{Proposition}

\theoremstyle{definition}
\newtheorem{definition}{Definition}
\newtheorem*{problem*}{Problem}

\newtheorem*{remark*}{Remark}

\newcommand{\calD}{{\cal D}}

\newcommand{\calF}{{\cal F}}

\newcommand{\calO}{{\cal O}}
\newcommand{\calP}{{\cal P}}
\newcommand{\calQ}{{\cal Q}}

\newcommand{\calS}{{\cal S}}

\newcommand{\calU}{{\cal U}}

\newcommand{\calX}{{\cal X}}
\newcommand{\calY}{{\cal Y}}

\newcommand{\bfp}{\mathbf{p}}
\newcommand{\bfq}{\mathbf{q}}

\newcommand{\bfx}{\mathbf{x}}

\newcommand{\bftheta}{\boldsymbol{\theta}}

\newcommand{\bbR}{\mathbb{R}}

\title{\LARGE \bf Learning Barrier Functions with Memory for Robust Safe Navigation}
   
\author{Kehan Long$^{*}$ \quad Cheng Qian$^{*}$ \quad Jorge Cort{\'e}s \quad Nikolay Atanasov%
\thanks{$^*$ indicates equal contribution.}%
\thanks{We gratefully acknowledge support from NSF RI IIS-2007141.}%
\thanks{The authors are with the Contextual Robotics Institute, University of California San Diego, La Jolla, CA 92093, USA. {\tt\small \{k3long,chqian,cortes,natanasov\}@ucsd.edu}}%
}

\begin{document}

% \setcitestyle{square}
\maketitle

% % %
% \marginJC{Shouldn't we move the "robust" adjective to "safe navigation" or something to that effect? The CBFs per se are not robust, is what we do with them (robust synthesis) that we should emphasize, no? Or something like "Learning Barrier Functions with Memory for Robust Controller Synthesis"}
% % %

\begin{abstract}
    Control barrier functions are widely used to enforce safety properties in robot motion planning and control. However, the problem of constructing barrier functions online and synthesizing safe controllers that can deal with the associated uncertainty has received little attention. This paper investigates safe navigation in unknown environments, using on-board range sensing to construct control barrier functions online. To represent different objects in the environment, we use the distance measurements to train neural network approximations of the signed distance functions incrementally with replay memory. This allows us to formulate a novel robust control barrier safety constraint which takes into account the error in the estimated distance fields and its gradient. Our formulation leads to a second-order cone program, enabling safe and stable control synthesis in a prior unknown environments.
%This leads to a second-order cone program formulation, enabling safe control and motion planning in unknown environments.
    %However, the safety properties become harder to enforce if robots are operated in some unknown environments; recently, constructing control barrier functions in real-time with the help of LiDAR/RGB-D measurements from the sensors equipped on the robots has become more and more popular. In our paper, we incorporate the LiDAR measurements and CBF to ensure that our robot traverses safely in unknown environments. To achieve our goal, firstly, we use the zero level-set points from the LiDAR measurements to train a neural network representation of the signed distance functions(SDF) over the environment. Next, based on the neural network representation of SDF, we can specify our control barrier function and its gradient and start navigating our robots. Lastly, we apply our methodology to a simulation work in Pybullet with a ground robot equipped with LiDAR in unknown environments and demonstrate the effectiveness of our approach. 
\end{abstract}

% % Two or three meaningful keywords should be added here
% \keywords{Signed Distance Function Estimation, Incrementally Trained Neural Network, Control Barrier Function for Safe Navigation, Second-Order Cone Programming}

\section{Introduction}

% High-level overview of the problem
% Give an high-level overview of our approach
% \KL{Title candidates:
% \begin{itemize}
%   \item Online Learning of Robust Barrier Functions for Safe Robot Navigation
%   \item Barrier Function Memory Replay: Robust Safety Constraints for Robot Navigation
%   \item Robust Barrier Function for Safe Navigation with Replay Memory
% \end{itemize}}
Modern applications of ground, aerial, and underwater mobile robots necessitate safe operation in unexplored and unstructured environments. Planning and control techniques that ensure safety, relying on limited field-of-view observations, are critical for autonomous navigation. The seminal work of Khatib \cite{potential-field} introduced artificial potential fields to enable collision avoidance during not only the motion planning stage but also the real-time control of a mobile robot. Potential fields inspired subsequent work on joint path-planning and control. Rimon and Koditschek \cite{navigation-function} developed navigation functions, special artificial potential functions designed to simultaneously guarantee collision avoidance and stabilization to a desired goal configuration. In the 2000's, barrier certificates were proposed as a general construction to verify safety of closed-loop nonlinear and hybrid systems \cite{barrier-certificate, barrier-hybrid}. Barrier certificates were extended by \cite{wieland2007} to consider control inputs explicitly and enable safe control design. Control barrier functions (CBFs) have become a key technique for encoding safety constraints, and have been successfully employed, along with control Lyapunov functions (CLFs), to encode stability requirements in quadratic program (QP) control synthesis for control-affine nonlinear systems \cite{cbf}.

A common assumption in many CLF-CBF-QP works is that the robot already has complete knowledge of the unsafe regions, encoded in an \emph{a priori known} CBF. However, navigation in unknown environments requires online estimation of unsafe regions using onboard sensing and, hence, a CBF should also be constructed online. This paper considers a mobile robot equipped with a LiDAR scanner and tasked to follow a motion plan, relying on streaming range measurements to ensure safety. %The trajectory tracking objective is specified using a CLF \cite{cbf}. 
We focus on constructing CBFs, approximating the shape of the objects in the environment, and introduce corresponding safety constraints in the synthesis of the robot's control policy to enforce safety.

% , either with 3D supervision \cite{deepsdf} or without 3D supervision \cite{Lin2020SDFSRNLS}.

\begin{figure}[t]
  \centering
  \includegraphics[width=1\linewidth]{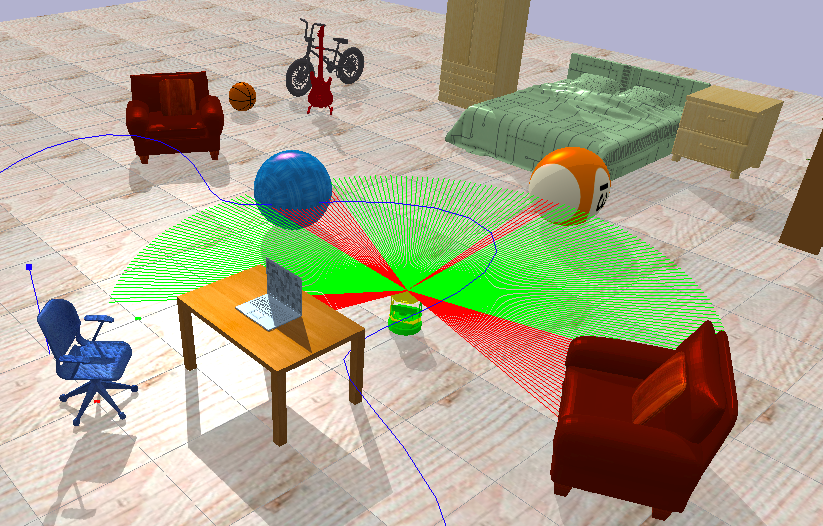}
  \caption{A ground robot equipped with a LiDAR scanner is navigating safely along a desired path (blue curve) in an unknown room.}
  \label{fig:1}
  \vspace*{-3ex}
\end{figure}

The signed distance function (SDF) of a set $\Omega$ in a metric space determines the distance of a given point $\boldsymbol{x}$ to the boundary of $\Omega$, with sign indicating whether $\boldsymbol{x}$ is in $\Omega$ or not. SDFs are an implicit surface representation, employed in computer vision and graphics for surface reconstruction and rendering \cite{deepsdf,Lin2020SDFSRNLS}. In contrast with other object geometry representations, an SDF provides distance and gradient information to object surfaces directly, which is necessary information for collision avoidance in autonomous navigation \cite{luxin2019fiesta}. In this work, we approximate the SDFs of observed objects and define corresponding CBF safety constraints for control synthesis. Using the LiDAR measurements, we incrementally train a multilayer perceptron to model each obstacle's SDF shape online. To balance the accuracy and efficiency of SDF reconstruction, we employ replay memory in training~\cite{prioritized}.  

The obstacle SDF approximations provide distance and gradient information for the barrier functions. We take into account the estimation errors for the controller synthesis, resulting in safety constraints where the input appears both linearly and within
an $l_2$-norm term. The presence of such constraints means that the control synthesis problem can no longer be stated as a QP but we show that the problem can be formulated as a second-order cone program (SOCP), which is still convex and can be solved efficiently. The proposed CLF-CBF-SOCP framework integrates the SDF estimation of the CBFs to synthesize  safe controls at each time instant, allowing the robot to safely navigate the unknown environment. In summary, this paper makes the following contributions. 

\begin{itemize}
    \item An online incremental training approach, utilizing replay memory, is proposed for deep neural network approximation of the signed distance functions of objects observed with streaming distance measurements.
    \item Our analysis shows that incorporating safety constraints that account for the worst-case estimation error of the SDF approximations into a control synthesis optimization problem leads to a convex SOCP formulation. 
\end{itemize}

\section{Related Work}
This section reviews related work on scene representations and CLF-CBF techniques for safe control.

Occupancy mapping techniques aim to partition an environment into occupied and free subsets using sensor observations. Octomap \cite{hornung2013octomap} is a probabilistic occupancy mapping algorithm that uses a tree-based structure to store and update the probability of occupied and free space based on streaming range observations. Occupancy information alone may not be enough for safe planning and control, since many algorithms require distance (or even gradient) information to the obstacle surfaces. Voxblox \cite{oleynikova2017voxblox} is an incremental truncated SDF mapping algorithm that approximates the distance to the nearest surface at a set of weighted support points. The approach also provides an efficient extension from truncated to complete (Euclidean) SDF using an online wavefront propagation algorithm.  
%using a BFS framework with the help of doubly linked lists for map maintenance, and two updating queues 
Fiesta \cite{luxin2019fiesta} further improves the efficiency and accuracy of building online SDF maps by using two independent queues for inserting and deleting obstacles separately and doubly linked lists for map maintenance. These SDF representations, however, require discretizing the environment into voxels, resulting in potentially large memory use. Recently, impressive results have been achieved in object shape modeling using deep neural networks. DeepSDF \cite{deepsdf} and Occupancy Nets \cite{occupancynet} implicitly represent 3D shapes by supervised learning via fully connected neural networks. Gropp \textit{et al}. \cite{implicitsurface} further extended the DeepSDF model by introducing the \textit{Eikonal} constraint that any SDF function must satisfy in training the neural network. 
% \NA{This description needs to be more precise: state that it is a fully connected network, that it is an extension of DeepSDF capturing the Eikonal constraint, and that only offline training has been considered (which can be part of the next sentence)}.
We extend this offline training method in two ways: enabling online training by introducing replay memory, which helps the network to remember the reconstructed shape from past observations, and adding truncated signed distance points to the training set to ensure the correct sign of the approximated SDF. Compared with state-of-the-art SDF methods \cite{oleynikova2017voxblox,luxin2019fiesta} for navigation and obstacle avoidance, which rely on discretization, our approach enables continuous and differentiable SDF representations, as required by the CBF framework.
% %
% \marginJC{When discussing other papers, I'm not a big fan of using "they", "their", vs "We", "ours". I prefer to just refer to the paper. Some authors/reviewers might take then they vs us as a bit confrontational.}
% %
% %
% \marginJC{From what I gather from [7], the construct estimated unsafe sets, then take them as ground truth in the safe controller design. We should make clear, here and in the statement of contributions, that here we not only look at the synthesis of CBFs, but also as to how take into account errors in their estimation in the controller synthesis to make sure design is safe and takes uncertainty properly into account.}
% %

% Wang \textit{et al}.~\cite{wang2018safe} considers uncertainty in the robot dynamics. In the controller design part, a Gaussian Process is used to represent the uncertainty in nonlinear dynamic systems. Then, the uncertainties are integrated into the barrier certificates to design a safety controller such that the robot is able to navigate safely even the dynamic model is not exactly known.
% Many CBF-based techniques for path following and obstacle avoidance~\cite{glotfelter2017nonsmooth,xu2017realizing} employ QP formulations to handle potential conflicts between the satisfaction of control objectives and safety constraints. The CLF acts as a tool to stabilize the system for achieving the control objectives whenever possible, whereas the CBF is used to ensure safety of the resulting controller.

Quadratic programming with CBF constraints offers an elegant and efficient framework for safe control synthesis in robot navigation tasks \cite{glotfelter2017nonsmooth,xu2017realizing}. This approach employs CLF constraints to stabilize the system and achieve the control objectives whenever possible, whereas the CBF constraints are used to ensure safety of the resulting controller at all times. CBFs are used for safe multi-robot navigation in \cite{glotfelter2017nonsmooth}. CLF and CBF constraints are combined to solve a simultaneous lane-keeping (LK) and adaptive speed regulation (ASR) problem in \cite{xu2017realizing}. CLFs are used to ensure convergence to the control objectives for LK and ASR, whereas CBFs are used to meet safety requirements. However, the barrier functions in these approaches are assumed to be known. When the environment is unknown, a robot may only rely on the estimation of barrier functions using its sensors. A closely related work by Srinivasan \textit{et al}. \cite{srinivasan2020synthesis} presents a Support Vector Machine (SVM) approach for the online synthesis of barrier functions from range data. By approximating the boundary between occupied and free space as the SVM decision boundary, the authors extract CBF constraints and solve a CBF-QP to generate safe control inputs. Our approach constructs CBFs directly from an approximation of the object SDFs and explicitly considers the potential errors in the CBF constraints when formulating the control synthesis optimization problem.

\section{Problem Formulation}
\label{sec:problem}

% Just an example of a reference
%In Sec.~\ref{sec:problem}

% Define the robot: state, dynamics, time

% Define the environment: how are the true objects represented: the environment is a union of two sets: free space and occupied space OR is a union of free space and n obstacle sets

% Define the sensing: relate the robot state/pose to the obstacle sets

% State the problem we are solving: go from point A to point B in an unknown environment while guaranteeing safety and stability along the way

% $\mathcal{X}$ is an $n$-dimensional state space, 
% $\mathcal{U}$ is an $m$-dimensional space of admissible controls, 
Consider a robot whose dynamics are governed by a non-linear control-affine system:
\begin{equation}
\label{eq: dynamic}
\begin{aligned}
    &\dot{\boldsymbol{x}} = f(\boldsymbol{x}) + g(\boldsymbol{x}) \boldsymbol{u} \\
    &  \boldsymbol{y} = v(\boldsymbol{x})
\end{aligned}
\end{equation}
where $\boldsymbol{x} \in \calX \subseteq \mathbb{R}^{n}$ is the robot state, $\boldsymbol{u} \in \mathcal{U} \subseteq \mathbb{R}^{m}$ is the control input, and $\boldsymbol{y} \in \calY \subseteq \mathbb{R}^{w}$ is the system output. Taking the robot in Sec. \ref{evaluation} as an example, $\boldsymbol{x}$ is the robot position and orientation, $\boldsymbol{u}$ is the linear and angular velocity input, while $\boldsymbol{y}$ is the robot position. Assume that $f : \mathbb{R}^{n} \mapsto \mathbb{R}^{n}$, $g : \mathbb{R}^{n} \mapsto \mathbb{R}^{n \times m}$, and $v : \mathbb{R}^{n} \mapsto \mathbb{R}^{w}$ are continuously differentiable functions. Define the admissible control input space as $\calU \coloneqq \{\boldsymbol{u} \in \mathbb{R}^m \,|\, A \boldsymbol{u} \leq \boldsymbol{b} \}$, where $A \in \mathbb{R}^{k \times m}$ and $\boldsymbol{b} \in \mathbb{R}^k$. The output space $\calY$ is a Euclidean robot workspace (e.g., robot positions), used for collision checking. The state space $\calX$ is partitioned into a closed safe set $\mathcal{S}$ and an open obstacle set $\mathcal{O}$ such that  $\mathcal{S} \cap \mathcal{O} = \emptyset$ and $\calX = \mathcal{S} \cup \mathcal{O}$. 
Correspondingly, the workspace $\calY$ is partitioned into a closed set $v(\mathcal{S})$ and an open set $v(\mathcal{O})$, obtained by applying the output function $v$ to each element in $\calS$ and $\calO$. 
The obstacle set $\mathcal{O}$ is further partitioned into $K$ obstacles, denoted $\mathcal{O}_1, \mathcal{O}_2, \ldots \mathcal{O}_K$, such that $\mathcal{O} = \cup_{i=1}^K \mathcal{O}_i$ and $\mathcal{O}_j \cap \mathcal{O}_k = \emptyset $ for any $0 \leq j,k \leq K$. Each of these sets is defined through a continuously differentiable function $\varphi_i: \mathbb{R}^{w} \mapsto \mathbb{R}$, with gradient $\nabla \varphi_i : \mathbb{R}^{w} \mapsto \mathbb{R}^{w}$. Formally, $\mathcal{O}_i = \{\boldsymbol{x} \in \mathcal{X}\, |\, \varphi_{i}(\boldsymbol{v(x)}) < 0 \}$. Note that $\varphi_i(\boldsymbol{v(x)}) < 0$ if the output $\boldsymbol{y}$ is in the open set $v(\mathcal{O}_i)$ and $\varphi_i(\boldsymbol{v(x)}) = 0$ if $\boldsymbol{y}$ is on its boundary $\partial v(\mathcal{O}_i)$.

The robot is equipped with a range sensor, such as a LiDAR scanner, and aims to follow a desired path, relying on the noisy distance measurements to avoid collisions. A \emph{path} is a piece-wise continuous function $r : [0,1] \mapsto \Int(v(\mathcal{S}))$.
% that maps the unit interval $\mathcal{I} = [0,1]$ to $\Int(\mathcal{S})$.
% The range sensor provides multiple depth measurements from the robot to the boundary of the obstacles at each time step $t_k, t_{k+1},\ldots$. Here is an example. "
Let $\calF(\boldsymbol{x}) \subset \calY$ be the field of view of the range sensor when the robot is in state $\boldsymbol{x}$. At discrete times $t_k$, for $k \in \mathbb{N}$, the range sensor provides a set of points $\calP_{k,i} = \crl{\bfp_{k,i,j}}_j \subset \calF(\boldsymbol{x}(t_k)) \cap \partial v(\calO_i)$ on the boundary of each obstacle $v(\calO_i)$ which are within its field-of-view. 
%Formally, for the $i$th obstacle, the range sensor returns a point cloud $\mathcal{P}_{k,i} = \{\boldsymbol{p}_{k,i,j}\}_j$ at time~$t_k$. 

% \NA{$m$ was already used for the dimension of the control input. It might be better to remove the explict mention of $m$ in eq. (1). We can also remove the use of $m$ here as well and free it up completely for a different person. We can keep the number of observed points not explicitly known, i.e., it is equal to $|\calP_{t,i}|$ if really needed.}
% %
% \marginJC{Later, we denote the points in the cloud as $\boldsymbol{p}_{t,i,j}$ instead of $p_{t,i,j}$. Which one is it? Also, we're using $k$ for number of obstacles, so use a different letter!}
% %

% %
% \marginJC{We don't specify frequency of availability of depth measurements. Is it $t=1,2,...$ or is it $t \in \mathbb{R}$?}
% %

%\begin{problem*}[Obstacle Learning and Control Synthesis for Safe Navigation]
%Consider the affine control system \eqref{eq: dynamic} operating in an unknown environment $\mathcal{X}$ with $k$ unsafe regions $\{\mathcal{O}_i\}_{i=1}^k \subset \mathcal{X}$. Given a desired path $r$ and noisy ranged sensor measurements $\mathcal{P}_{t,i}$ for each time $t \geq 0$, learn approximations $\{\mathcal{\tilde O}_i \}_{i=1}^k$ of the unsafe regions and employ them to design a feedback controller $\boldsymbol{x} \mapsto \boldsymbol{u} (\boldsymbol{x})$ that ensures the robot can safely move along the path~$r$.
%\end{problem*}

\begin{problem*}
Consider the system in~\eqref{eq: dynamic} operating in an unknown environment, i.e., the obstacles sets $\crl{v(\calO_i)}_{i=1}^K$ are unknown a priori. Given a desired path $r$ and noisy range sensor measurements $\mathcal{P}_{k,i}$ for $k \in \mathbb{N}$ and $i = 1,\ldots, K$, design a feedback controller $\boldsymbol{u}(\boldsymbol{x})$ for \eqref{eq: dynamic} that ensures that the robot can move safely along the path $r$, i.e., $\boldsymbol{y}(t) \in v(\calS)$ for all $t$ and $\boldsymbol{y}(t) \to r(1)$ as $t \to \infty$.
\end{problem*}

% Define the robot: state, dynamics, time

% Define the environment: how are the true objects represented: the environment is a union of two sets: free space and occupied space OR is a union of free space and n obstacle sets

% Define the sensing: relate the robot state/pose to the obstacle sets

% State the problem we are solving: go from point A to point B in an unknown environment while guaranteeing safety and stability along the way

\section{Obstacle Estimation via Online Signed Distance Function Approximation}\label{sec:estimation}

Throughout the paper, we rely on the concept of signed distance function to describe each unsafe region $\{v(\mathcal{O}_i)\}_{i=1}^K$. For each $i$, the SDF function $\varphi_i: \mathcal{Y} \mapsto \mathbb{R}$ is:
\begin{equation}
\varphi_{i}(\boldsymbol{y}) := \begin{cases}
 -d(\boldsymbol{y},\partial v(\calO_i)), & \boldsymbol{y} \in v(\calO_i), \\
 \phantom{-} d(\boldsymbol{y},\partial v(\calO_i)), &  \boldsymbol{y} \notin v(\calO_i),
 \end{cases}
\end{equation} 
where $d$ denotes the Euclidean distance between a point and a set. In this section, we describe an approach to construct approximations %$\tilde{\varphi}_{k,i}$
% \NA{Should this have a subscript $k$ as well, i.e., $\tilde{\varphi}_{k,i}$? Even more accurate would be to introduce the parameterization here: $\tilde{\varphi}_i(\boldsymbol{x};\bftheta_k)$ and use a subscript $k$ for the parameters $\bftheta_k$.} 
to the obstacle SDFs $\varphi_i$ using the point cloud observations $\{\calP_{k,i}\}_k$ at time step $t_k$.

%=====================================================

\subsection{Data Pre-processing}
\label{sec:data}
% To simplify the notation, we suppose that for each obstacle, the training starts at $t=t_0$\NA{This sentence does not add any useful information.}. 

Given the point cloud $\mathcal{P}_{k,i} \subset \mathbb{R}^w$ on the surface of the $i$th obstacle at time $t_k$, we can regard $\mathcal{P}_{k,i}$ as the points on the zero level-set of a distance function. To normalize the scale of the training data, we define the point coordinates with respect to the obstacle centroid. Since the centroid is unknown, we approximate it as the sample mean of the training points $\bar{\bfp}_{k,i} := \frac{1}{m}\sum_{j=1}^{m} \bfp_{k,i,j}$.
% $\bar{\bfp}_{k,i} := \arg\min_{\bfp} \sum_j \|\bfp_{k,i,j} - \bfp\|^2$. 
The points $\bfp_{k,i,j} - \bar{\bfp}_{k,i}$ are centered around $\bar{\bfp}_{k,i}$ and have a measured distance of $0$ to the obstacle surface $v(\partial \calO_i)$.
% \NA{This sentence defines 3 different things that are not necessarily related. I suggest splitting it into 3 sentences. Explain a little more what we are trying to do. Why are we defining $\bar{\bfp}_{k,i}$ or $\bfq_{k,i,j}$?}. 
Let $c = \|v(\boldsymbol{x}_k)-\bfq_{k,i,j}\|$ be the Euclidean distance between the robot state $\boldsymbol{x}_k := \boldsymbol{x}(t_k)$ and $\bfp_{k,i,j}$. Let $\delta > 0$ be a small positive constant. Define a point $\bfq_{k,i,j} := \frac{\delta}{c} v(\boldsymbol{x}_k) + (1-\frac{\delta}{c}) \bfp_{k,i,j}$ along the LiDAR ray from the output $v(\boldsymbol{x}_k)$ to $\bfp_{k,i,j}$ that is approximately a distance $\delta$ from the obstacle surface $\partial v(\calO_i)$. We call the set $\crl{\bfq_{k,i,j} - \bar{\bfp}_{k,i}}_j$ a truncated SDF point set. The training set for each obstacle $i$ is constructed as a union of the points on the boundary and the truncated SDF points. The training set at time $t_k$ is $\calD_{k,i} := \crl{ (\bfp_{k,i,j} - \bar{\bfp}_{k,i}, 0)} \cup \crl{ (\bfq_{k,i,j} - \bar{\bfp}_{k,i}, \delta)}$.

\subsection{Loss Function}
\label{sec: loss}

% = \tilde{\varphi}_{k,i}

Inspired by the recent impressive results on multilayer perceptron approximation of SDF~\cite{deepsdf,implicitsurface}, we introduce a fully connected neural network $\tilde{\varphi}_i(\boldsymbol{y};\bftheta_k)$ with parameters $\bftheta_k$ to approximate the SDF of each observed obstacle $i$ at time step $t_k$. Our approach relies on the fact that the norm of the SDF gradient satisfies the \emph{Eikonal} equation $\norm{\nabla {\varphi}_i(\boldsymbol{y}) } = 1$ in its domain. We use a loss function that encourages $\tilde{\varphi}_i(\boldsymbol{y};\bftheta_k)$ to approximate the measured distances in a training set $\calD$ (distance loss $\ell_i^D$) and to satisfy the Eikonal constraint for set of points $\calD'$ (Eikonal loss $\ell_i^E$). For example, $\tilde{\varphi}_i$ equals $0$ for points on the obstacle surface and equals $\delta$ for the truncated SDF points along the LiDAR rays. The loss function is defined as $\ell_i(\bftheta_k; \calD,\calD') := \ell_i^{D}(\bftheta_k;\calD) + \lambda \ell_i^{E}(\bftheta_k; \calD')$, with a parameter $\lambda > 0$ and:
\begin{equation}
\label{eq:loss}
\begin{aligned}
\ell_i^{D}(\bftheta_k;\calD) &:= \frac{1}{|\calD|} \sum_{(\bfp,d) \in \calD} |\tilde{\varphi}_i(\boldsymbol{p};\bftheta_k) - d|,\\
\ell_i^{E}(\bftheta_k;\calD') &:= \frac{1}{|\calD'|} \sum_{\bfp \in \calD'} \prl{\|\nabla \tilde{\varphi}_i(\boldsymbol{p}; \bftheta_k)\|-1}.
\end{aligned}
\end{equation}
% The distance loss term $\ell_i^{D}(\bftheta;\calD)$ seeks to have $\tilde{\varphi}_i$ equal the measured distances obtained from the lidar data as described in Sec.~\ref{sec:data}. For example, $\tilde{\varphi}_i$ equals $0$ for points on the obstacle surface and equals $\delta_d$ for the truncated SDF points $\bfq_{k,i,j}$ along the LiDAR rays. The \textit{Eikonal} loss $\ell_i^{E}(\bftheta;\calD')$ encourages $\norm{\nabla_{\boldsymbol{x}} \tilde \varphi_i(\boldsymbol{x};\bftheta)}$ to be $1$.\NA{These 2-3 sentences after \eqref{eq:loss} seem repetitive with what was described already and may be removed or moved earlier and merged with the discussion above \eqref{eq:loss}.}
%
The training set $\calD'$ for the Eikonal term can be generated arbitrarily in $\calY$ since $\tilde{\varphi}_i$ needs to satisfy the Eikonal constraint \textcolor{blue}{\emph{almost everywhere} in $\calY$}. In practice, we generate $\calD'$ by sampling points from a mixture of a uniform distribution and Gaussians centered at the points in the distance training set $\calD$ with standard deviation equal to the distance to the $k$-th nearest neighbor ($k = {| \mathcal{D}|}/{2}$). For the distance training set $\calD$, ideally, we should use as much data as possible to obtain an accurate approximation of the SDF $\varphi_i$. For example, at time $t_k$, all observed data may be used as the distance training set $\calD = \cup_{l=0}^k \calD_{l,i}$. However, the online training time becomes longer and longer as the robot receives more and more LiDAR measurements, which makes it impractical for real-time mapping and navigation tasks. We introduce an \say{Incremental Training with Replay Memory} approach in Sec.~\ref{sec: incremental}, which obtains accurate SDF estimates with training times suitable for online learning.

\subsection{Incremental Learning}
\label{sec: incremental}
% %
% \marginJC{In this section, we seem to be describing an algorithm "as we go" with the exposition. I wonder if a simple algorithmic environment with pseudocode wouldn't be easier to understand/follow by the reader}
% %

When an obstacle $i$ is first observed at time $t_0$, we use the data set $\mathcal{D}_{0,i}$ to train the SDF network $\tilde \varphi_i(\boldsymbol{y};\bftheta_0)$. When new observations $\mathcal{D}_{k,i}$ of obstacle $i$ are obtained at $k = 1,2,\dots$, we need to update the network parameters $\bftheta_k$. We first consider an approach that updates $\tilde \varphi_i(\boldsymbol{y};\bftheta_{k-1})$ based on the new data set $\mathcal{D}_{k,i}$ and uses the previous parameters $\bftheta_{k-1}$ as initialization. We call this approach \say{Incremental Training} (IT). Alternatively, we can update $\tilde \varphi_i(\boldsymbol{y};\bftheta_{k-1})$ using all prior data $\calD_i = \cup_{l=0}^k \calD_{l,i}$, observed up to time $t_k$, and $\bftheta_{k-1}$ as initialization. We call this approach \say{Batch Training} (BT).

% \NA{I suggest defining the 3 approaches we consider later formally here using the notation $\ell_i(\bftheta;\calD,\calD')$ I introduced earlier: Training with all data, Training with newest data, Training with memory reply. You can even think of more succinct names for the three method so we can refer to them later.}  
%After training the network $\tilde \varphi_i(\boldsymbol{y};\bftheta_0)$ when the obstacle $i$ was first observed at time step $t = t_0$ with the training set $\mathcal{D}_{0,i}$,
%
% \KL{Should we say $t = t_k$ for some time step $k$ here to be more general? I am not sure, because we only have the training set for time $t_0$ in sec.A and the following training sets will be constructed by replay memory}
%
%we need to consider how to update the MLP parameters $\bftheta_k$ when we receive new observation sets $\mathcal{D}_{k,i}$  for $k = 1,2,\dots$. We first consider an approach that updates $\tilde \varphi_i(\boldsymbol{y};\bftheta_{k-1})$ based on the new observing set $\mathcal{D}_{k,i}$ and uses the previous MLP parameters $\bftheta_{k-1}$ as initialization, we call it \say{Incremental Training} (IT) approach. Alternatively, we consider updating $\tilde \varphi_i(\boldsymbol{y};\bftheta_{k-1})$ based on all observing sets $\calD_i = \cup_{l=0}^k \calD_{l,i}$ up to time $t_k$ and uses $\bftheta_{k-1}$ as initialization. We call this update method as \say{Batch Training} (BT) approach. 

The IT approach is efficient because, at each time $t_k$, it uses data sets $\calD_{k,i}$ of approximately constant cardinality, leading to approximately constant update times during online training. However, discarding the old data $\cup_{l=0}^{k-1} \calD_{l,i}$ and using stochastic gradient descent to re-train the network parameters $\bftheta_{k-1}$ on the new data $\calD_{k,i}$ causes degradation in the neural network's ability to represent the old data. In other words, the SDF approximation $\tilde \varphi_i(\boldsymbol{y};\bftheta_k)$ at time $t_k$ is good at approximating the latest observed obstacle surface but the approximation quality degrades at previously observed surfaces, as shown in Fig. \ref{fig:replay_result}. The BT approach does not have this limitation since it uses all data $\cup_{l=0}^k \calD_{l,i}$ for training at time $t_k$ but, as a result, its training time increases over time. Hence, our motivation for introducing an \say{Incremental Training with Replay Memory} (ITRM) approach is to balance the trade-off between SDF estimation error and online training time, making it suitable for real-time robotic tasks. 

%One issue with ILT approach is the network trained by stochastic gradient descent rapidly degrade on old tasks when trained successively with new data. The issue with BT approach is the training time increases as more observing data emerges. Therefore, the motivation for introducing the \say{Incremental Training with Replay Memory}(ILTRM) approach is to balance the error in estimation while ensure the training time is suitable for real-time robotic tasks. 

% Using ideas from experience replay to reduce the forgetting of the neural networks \cite{experiencereplay}, w

Experience replay is an effective and commonly used technique for online reinforcement learning, which enables convergence of stochastic gradient descent to a critical point in policy and value function approximation \cite{experiencereplay,prioritized}. This idea has not been explored for online supervised learning of geometric surfaces. The first contribution of this paper is to use replay memory for online incremental learning of SDF. At each time step $t_{k-1}$, We construct an experience replay memory $\mathcal{Q}_{k-1,i}$ at each time step by utilizing the SDF approximations $\tilde{\varphi}_i(\boldsymbol{y};\bftheta_{k-1})$.

\begin{figure}[t]
\centering
\subcaptionbox{Training data at $k=0$ \label{fig:2a}}{\includegraphics[width=0.33\linewidth]{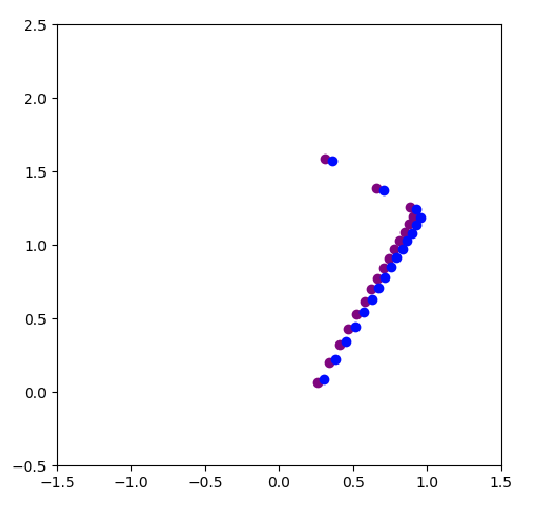}\hspace{0.5ex}}%
\hfill%
\subcaptionbox{IT data at $k=70$ \label{fig:2b}}{\includegraphics[width=0.33\linewidth]{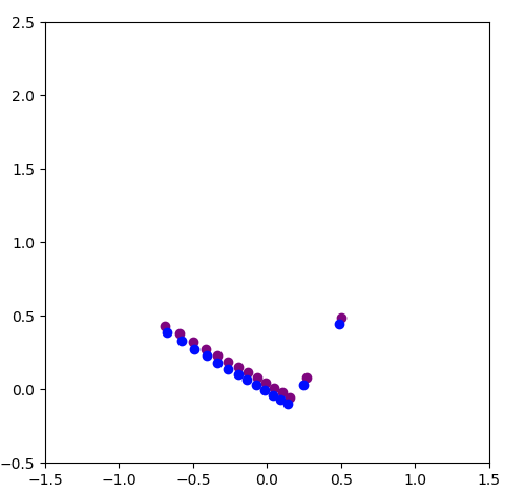}}%
\hfill% 
\subcaptionbox{ITRM data at $k=70$ \label{fig:2c}}{\includegraphics[width=0.33\linewidth]{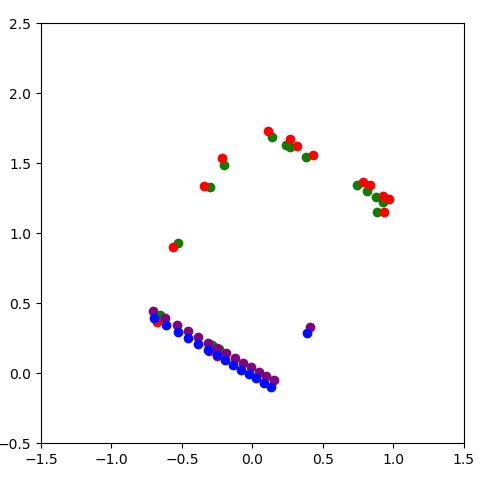}}

\subcaptionbox{Estimation at $k=0$ \label{fig:2d}}{\includegraphics[width=0.33\linewidth]{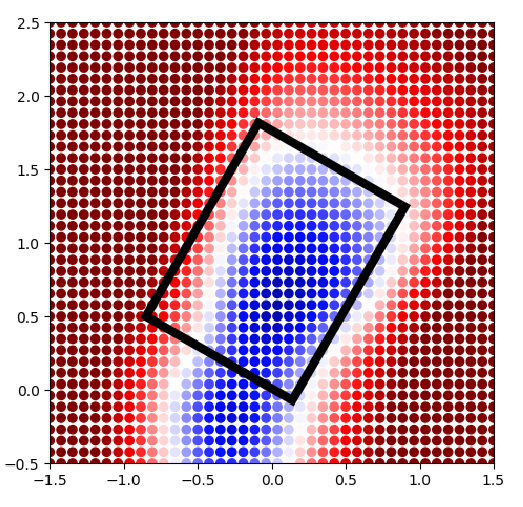}}%
\hfill%
\subcaptionbox{IT estimation at \\$k=70$ \label{fig:2e}}{\includegraphics[width=0.33\linewidth]{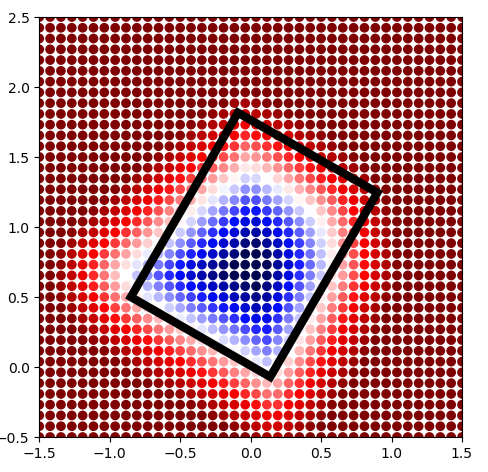}}%
\hfill%
\subcaptionbox{ITRM estimation at \\$k=70$ \label{fig:2f}}{\includegraphics[width=0.33\linewidth]{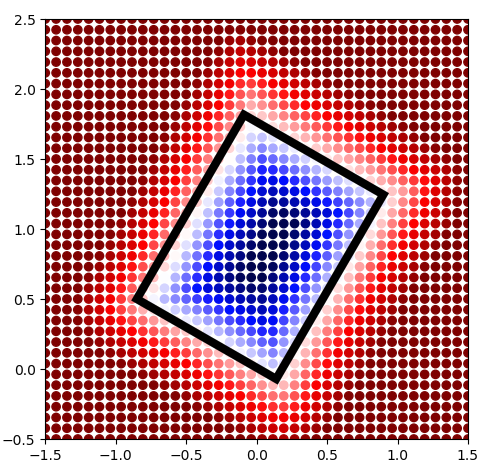}}
\caption{Shape estimation with and without replay memory. The top row shows the training data used at time step $k=0$ and $k=70$ by  the IT and ITRM approaches. The purple points are the observed LiDAR end points, while the blue points are the truncated SDF points along the LiDAR rays. In (c), the green and red points are boundary and truncated SDF points obtained from the replay memory. The bottom row shows the SDF estimation of the obstacle surface at time step $k=0$ and $k=70$ for the IT and ITRM approaches. The black rectangle shows the ground-truth obstacle boundary, while colored regions are level-sets of the SDF estimate. The white region denotes the estimated obstacle boundary. The blue (resp. red) region denotes negative (resp. positive) signed distance. IT and ITRM use the same data and lead to the same estimate at $k=0$ because the replay memory set is empty. In (e), the SDF estimate of the top obstacle region at $k = 70$,~without~memory replay, degrades compared to (d). In (f), training with replay~memory helps the neural network remember the overall obstacle shape.}
\label{fig:replay_result}
\vspace*{-3ex}
\end{figure}

\begin{definition}
The \emph{replay memory} $\calQ$ associated with a signed distance function $\varphi$ and truncation parameter $\tau \geq 0$ is the set of points that are at most a distance $\tau$ from the zero-level set of $\varphi$: $\calQ := \crl{(\bfq,\varphi(\bfq)) \in (\mathbb{R}^w,\mathbb{R}) \mid |\varphi(\bfq)| \leq \tau}$.
%with corresponding distances to the zero-level set of $\varphi$ $\calQ := \crl{(\bfq,d) \in (\calX,R) \mid |\tilde{\varphi}(\bfq)| \leq \tau}$ that are at most a distance $\tau$ from the zero-level set of $\tilde{\varphi}$.
\end{definition}

To construct the replay memory set $\mathcal{Q}_{k-1,i}$, we need to generate samples from the level sets of the SDF approximation $\tilde{\varphi}_i(\boldsymbol{y};\bftheta_{k-1})$. 
%Matan \textit{et al}. \cite{atzmon2019controlling} introduced a simple and scalable approach to sample from neural network level sets.
In robotics applications, where the environments are commonly 2-D or 3-D, samples from the level sets of $\tilde{\varphi}_i(\boldsymbol{y};\bftheta_{k-1})$ can be obtained using the Marching Cubes algorithm~\cite{Lorensen87marchingcubes}. In our experiments in Sec.~\ref{evaluation}, we used Marching Cubes to extract samples $\bfq_0$ and $\bfq_\delta$ from the zero and $\delta$ level-sets of $\tilde{\varphi}_i(\boldsymbol{y};\bftheta_{k-1})$, respectively, and construct the replay memory as $\mathcal{Q}_{k-1,i} := \crl{ (\bfq_0, 0)} \cup \crl{ (\bfq_\delta, \delta)}$.

%In practical, for $2D$ and $3D$ environments, we could also extract zero level-set  and some level-set with positive value $\delta_d$ of the previous $\tilde \varphi_i(\boldsymbol{x};\bftheta_{k-1})$ using the Marching Squares/Cubes algorithm~\cite{Lorensen87marchingcubes} on a uniform sampled grid of size $l^{2,3}$. We then construct the replay memory $\mathcal{Q}_{k,i} := \crl{ (\bfq_0, 0)} \cup \crl{ (\bfq_\delta, \delta_d)}$, where $\bfq_0$ and $\bfq_\delta$ are the points on the two extracted level-set respectively.
%  \NA{I suggest stating this a bit more generally (e.g., what if we are in 3-D or N-D) and more formally (should we have a definition of sdf reply memory? For example: ``The replay memory associated with an SDF representation $\tilde{f}$ is a set of points $\calQ := \crl{\bfq \in \calX \mid |\tilde{f}(\bfq)| \leq \tau}$ that are at most a distance $\tau$ from the zero-level set of $\tilde{f}_i$.'' We can say that in practice, we can obtain samples from the set $\calQ$ using the marching cubes algorithm or is there a way to sample directly from the neural network level set, e.g., \url{https://arxiv.org/abs/1905.11911}.}

Given the replay memory $\mathcal{Q}_{k-1,i}$, our ITRM approach constructs a training set at time $t_{k}$ by combining the latest observation $\mathcal{D}_{k,i}$ with a randomly sampled subset $\bar{\calQ}_{k-1,i}$ from $\mathcal{Q}_{k-1,i}$. To make the algorithm efficient without losing significant information about prior data, we let $|\bar{\calQ}_{k-1,i}| = |\mathcal{D}_{k,i}|$, i.e., we pick as many points from the replay memory as there are in the latest observation $\mathcal{D}_{k,i}$. 

%After constructing the replay memory $\mathcal{Q}_{k,i}$, we randomly pick point set $\bar{\calQ}_{k,i}$ from the replay memory and combine them with the current observing set $\mathcal{D}_{k,i}$ to obtain our new training data for time step $t_k$. To make the algorithm efficient without losing much information about previous data, we let $|\bar{\calQ}_{k,i}| = |\mathcal{D}_{k,i}| $, which means that we pick the same number of points as the current observing set from the replay memory. 

The three training techniques, IT, BT, and ITRM, discussed in this section, are formally defined as follows. In all approaches, for obstacle $i$, the parameters $\bftheta_{k-1}$ obtained at the previous time step are used as an initial guess for the new parameters $\bftheta_{k}$ at time $t_k$. The SDF approximations $\tilde \varphi_i(\boldsymbol{y};\bftheta_{k})$ are trained via the loss function in \eqref{eq:loss} but each method uses a different training set $\calD$:
\[
\text{IT: } \calD = \calD_{k,i} \quad \text{BT: } \calD = \cup_{l=0}^k \calD_{l,i} \quad \text{ITRM: } \calD = \mathcal{D}_{k,i} \cup \bar{\calQ}_{k,i}.
\]
We use the continuously differentiable Softplus function $\ln(1+e^x)$ as the non-linear activation to ensure that $\tilde \varphi_i(\boldsymbol{y};\bftheta_k)$ is continuously differentiable. The gradient $\nabla \tilde{\varphi}_i(\boldsymbol{y};\bftheta_k)$ is an input in the loss function in \eqref{eq:loss} and can be calculated via backpropagation \cite{implicitsurface}.

\section{Safe Navigation with Estimated Obstacles}\label{sec:navigation}
We rely on the estimated SDFs constructed in Sec.~\ref{sec:estimation} to formalize the synthesis of a controller that guarantees safety with respect to the exact obstacles, despite errors in the SDF approximation. Our analysis assumes error bounds are available, and we leave their actual computation for future work. In this regard, several recent works study the approximation power and error bounds of neural networks~\cite{Bailey2019ApproximationPO,Yarotsky2017ErrorBF}. In our evaluation in Sec.~\ref{evaluation}, we obtain SDF error bounds by comparing the estimated and ground-truth object SDFs.

\subsection{Safe Control with Estimated Barrier Functions}

%%\marginJC{Here CBF depends on $x$, not $y$?}
%

%Here we describe our approach to ensure that the robot state remains in the safe set $\mathcal{S}$ throughout its evolution. 
A useful tool to ensure that the robot state remains in the safe set $\mathcal{S}$ throughout its evolution is the notion of control barrier function (CBF).

\begin{definition}[{\cite{ames2016control}}]
A continuously differentiable function $h: \mathbb{R}^n \mapsto {\mathbb{R}}$, with $h(\boldsymbol{x}) >0$ if $\boldsymbol{x} \in \operatorname{Int}(\mathcal{S})$ and $h(\boldsymbol{x}) =0$ if $\boldsymbol{x} \in \partial \mathcal{S}$, is a \emph{zeroing control barrier function (ZCBF)} on $\mathcal{X} \subset \mathbb{R}^n$ if there exists an extended class $\mathcal{K}$ function $\alpha_h$ such that, for each $\boldsymbol{x} \in \mathcal{X}$, there exists $\boldsymbol{u} \in \mathcal{U}$ with
\begin{equation}\label{eq:cbf}
    \mathcal{L}_f h(\boldsymbol{x}) + \mathcal{L}_g h(\boldsymbol{x})\boldsymbol{u} + \alpha_h (h(\boldsymbol{x})) \geq 0,
\end{equation}
where $\mathcal{L}_f h(\boldsymbol{x})$ is the Lie derivative of $h(\boldsymbol{x})$ along $f(\boldsymbol{x})$.
\end{definition}

Any Lipschitz-continuous controller $\boldsymbol{u}: \mathcal{X} \mapsto \mathcal{U}$ such that $\boldsymbol{u}(\boldsymbol{x})$ satisfies \eqref{eq:cbf} renders the set $\mathcal{S}$ forward invariant for the control-affine system \eqref{eq: dynamic} \cite{cbf, ames2016control}. If the exact SDFs $\varphi_i$ describing the obstacles $\mathcal{O}_i$ were known, we could define ZCBFs $h_i(\boldsymbol{x}) := \varphi_i(v(\boldsymbol{x}))$ that ensure the forward invariance of $\calS$. However, when the environment is observable through online range measurements as described in Sec.~\ref{sec:estimation}, we only have the estimated SDFs $\tilde \varphi_i(v(\boldsymbol{x});\bftheta_k)$ at our disposal to define $\tilde{h}_i(\boldsymbol{x}) := \tilde{\varphi}_i(v(\boldsymbol{x});\bftheta_k)$. Our next result describes how to use this information to ensure the safety of $\mathcal{S}$. The statement is for a generic $\tilde{h}$ (e.g., a single obstacle). We later particularize our discussion to the case of multiple obstacles.

\begin{proposition}\label{prop:well-defined}
Let $e_1(\boldsymbol{x}) := h(\boldsymbol{x}) -\tilde{h}(\boldsymbol{x}) \in \mathbb{R}$ be the error at $\boldsymbol{x}$ in the approximation of $h$, and likewise, let $e_2(\boldsymbol{x}) := \nabla h(\boldsymbol{x}) - \nabla \tilde{h}(\boldsymbol{x}) \in \mathbb{R}^n$ be the error at $\boldsymbol{x}$ in the approximation of its gradient. Assume there are available known functions $e_h(\boldsymbol{x}) : \mathbb{R} \mapsto \mathbb{R}_{\geq 0}$ and $e_{\nabla h}(\boldsymbol{x}): \mathbb{R}^n \mapsto \mathbb{R}_{\geq 0}$ such that
\begin{equation}
\label{eq: error_bound}
\begin{aligned}
    |e_1 (\bfx)| \leq e_h(\boldsymbol{x}) , \quad 
    \|e_2 (\bfx) \|  \leq e_{\nabla h}(\boldsymbol{x})
\end{aligned}
\end{equation}
with $e_h(\boldsymbol{x}) \rightarrow 0$ and $e_{\nabla h}(\boldsymbol{x}) \rightarrow 0$ as $\boldsymbol{x} \rightarrow \partial\mathcal{S}$. Let
\begin{align}
    \mathcal{K}_{\tilde{h}}(\boldsymbol{x}) &:= \{\boldsymbol{u} \in \mathcal{U} \mid \mathcal{L}_{f} \tilde{h}(\boldsymbol{x}) + \mathcal{L}_{g} \tilde{h}(\boldsymbol{x}) \boldsymbol{u}  -
     \label{eq: new_set_h}
    \\
    & \; \|f(\boldsymbol{x}) + g(\boldsymbol{x})\boldsymbol{u}\|e_{\nabla h}(\boldsymbol{x}) + \alpha_h(\tilde{h}(\boldsymbol{x}) - e_h(\boldsymbol{x})) \geq 0\} .
    \notag 
\end{align}
%is non-empty for all $\boldsymbol{x} \in \mathcal{S}$.
Then, any locally Lipschitz continuous controller $\boldsymbol{u}: \mathcal{X} \mapsto \mathcal{U}$ such that $\boldsymbol{u}(\boldsymbol{x}) \in \mathcal{K}_{\tilde{h}}(\boldsymbol{x})$ guarantees that the safe set $\mathcal{S}$ is forward invariant.
\end{proposition}

\begin{proof}
We start by substituting $h(\boldsymbol{x}) = \tilde{h}(\boldsymbol{x}) + e_1(\boldsymbol{x})$ in \eqref{eq:cbf}:
\begin{multline*}
\mathcal{L}_{f} \tilde{h}(\boldsymbol{x}) + \mathcal{L}_{g} \tilde{h}(\boldsymbol{x}) \boldsymbol{u} + e_2(\boldsymbol{x})^\top f(\boldsymbol{x}) + e_2(\boldsymbol{x})^\top g(\boldsymbol{x}) \boldsymbol{u} \\\geq -\alpha_h(\tilde{h}(\boldsymbol{x}) + e_1(\boldsymbol{x}) ).
\end{multline*}
%We start by adding the error terms to~\eqref{eq:cbf} as follows: $(\nabla \tilde{h} + e_2) (f(\boldsymbol{x}) + g(\boldsymbol{x}))
%     = \mathcal{L}_{f} \tilde{h}(\boldsymbol{x}) + \mathcal{L}_{g} \tilde{h}(\boldsymbol{x}) \boldsymbol{u} + \langle\ e_2 , f(\boldsymbol{x})\rangle + \langle\ e_2 , g(\boldsymbol{x}) \rangle \boldsymbol{u}$ on the left-hand side, and $-\alpha_h(h) = -\alpha_h(\tilde{h}+ e_1)$ on the right-hand side.
% \begin{align*}
%      (\nabla \tilde{h} + e_2) (f(\boldsymbol{x}) + g(\boldsymbol{x}))
%      &= \mathcal{L}_{f} \tilde{h}(\boldsymbol{x}) + \mathcal{L}_{g} \tilde{h}(\boldsymbol{x}) \boldsymbol{u} + \langle\ e_2 , f(\boldsymbol{x})\rangle + \langle\ e_2 , g(\boldsymbol{x}) \rangle \boldsymbol{u} \\
%      -\alpha_h(h) &= -\alpha_h(\tilde{h}+ e_1)
% \end{align*}
For any fixed $\boldsymbol{x}$ and any errors $e_1(\boldsymbol{x})$ and $e_2(\boldsymbol{x})$ satisfying \eqref{eq: error_bound}, we need the minimum value of the left-hand side greater than the maximum value of the right-hand side to ensure that \eqref{eq:cbf} still holds, namely:
\begin{equation}
\label{eq:errorh}
\scaleMathLine[0.9]{\begin{aligned}
    \min_{\|e_{2}(\boldsymbol{x})\| \leq e_{\nabla h}(\boldsymbol{x})} \Big\{ \mathcal{L}_{f} \tilde{h}(\boldsymbol{x}) + \mathcal{L}_{g} \tilde{h}(\boldsymbol{x}) \boldsymbol{u} + e_2(\boldsymbol{x})^\top f(\boldsymbol{x}) + & \\e_2(\boldsymbol{x})^\top g(\boldsymbol{x}) \boldsymbol{u}\Big\}\geq \max_{|e_{1}(\boldsymbol{x})| \leq e_{h}(\boldsymbol{x})} \Big\{-\alpha_h(\tilde{h}(\boldsymbol{x}) + e_1(\boldsymbol{x}))\Big\}.& 
\end{aligned}}
\end{equation}
Note that since $e_h(\boldsymbol{x}) \geq 0$ and $\alpha_h$ is an extended class $\mathcal{K}$ function, the maximum value in \eqref{eq:errorh} is obtained by:
\begin{equation*}
    \max_{|e_{1}(\boldsymbol{x})| \leq e_{h}(\boldsymbol{x})} -\alpha_h(\tilde{h}(\boldsymbol{x}) + e_1((\boldsymbol{x})))= -\alpha_h(\tilde{h}(\boldsymbol{x}) - e_h(\boldsymbol{x})).
\end{equation*}
The minimum value is attained when $e_2(\boldsymbol{x})$ is in the opposite direction to the gradient of $f(\boldsymbol{x}) + g(\boldsymbol{x})\boldsymbol{u}$, namely
\begin{align*}
    \min_{\|e_{2}(\boldsymbol{x})\| \leq e_{\nabla h}(\boldsymbol{x})} \Big\{ e_2(\boldsymbol{x})^\top f(\boldsymbol{x}) + & e_2(\boldsymbol{x})^\top g(\boldsymbol{x}) \boldsymbol{u}\Big\} = \\ &-\|f(\boldsymbol{x}) + g(\boldsymbol{x})\boldsymbol{u}\|e_{\nabla h}(\boldsymbol{x}).
\end{align*}
Therefore, the inequality condition~\eqref{eq:errorh} can be rewritten as
$   \mathcal{L}_{f} \tilde{h}(\boldsymbol{x}) + \mathcal{L}_{g} \tilde{h}(\boldsymbol{x}) \boldsymbol{u} - \|f(\boldsymbol{x}) + g(\boldsymbol{x})\boldsymbol{u}\|e_{\nabla h}(\boldsymbol{x}) \geq -\alpha_h(\tilde{h}(\boldsymbol{x}) - e_h(\boldsymbol{x}))
$, which is equivalent to the condition in~\eqref{eq: new_set_h}.
\end{proof}

Proposition~\ref{prop:well-defined} allows us to synthesize safe controllers even though the obstacles are not exactly known, provided that error bounds on the approximation of the barrier function and its gradient are available and get better as the robot state gets closer to the boundary of the safe set $\mathcal{S}$.

% \subsection{CBF Constraint for Unknown Robot Dynamics}
% If the robot dynamics is unknown, we can modify the CBF condition and take the error in robot dynamics into account:
% \begin{equation}
%     \min_{\|e_{f}\| \leq \beta_{f} \sigma_{f}, \|e_{g}\| \leq \beta_{g} \sigma_{g}} \bigg\{ \mathcal{L}_{\tilde{f}+e_f} h(\boldsymbol{x}) + \mathcal{L}_{\tilde{g}+e_g} h(\boldsymbol{x}) \boldsymbol{u}\bigg\} \geq -\alpha(h(\boldsymbol{x}))
% \end{equation}
% which can also ensure (\ref{eq: zbf}) is satisfied. This modification is making the controller much "safer" even if the robot dynamics is unknown. With this in mind, we can modify the admissible set of control inputs as:

% \begin{equation}\nonumber
% \begin{aligned}
% \mathcal{K}_{V}(x) &= \{\boldsymbol{u}\in U : \mathcal{L}_{\tilde{f}} V(x) - \|\nabla V(\boldsymbol{x})\|\beta_{f} \sigma_{f} + (\mathcal{L}_{\tilde{g}} V(x) \pm \|\nabla V(\boldsymbol{x})\|\beta_{g} \sigma_{g} )\boldsymbol{u} + \alpha (V(x)) \leq 0\}
% \\
% \mathcal{K}_{h}(x) &= \{\boldsymbol{u}\in U : \mathcal{L}_{\tilde{f}} h(x) - \|\nabla h(\boldsymbol{x})\|\beta_{f} \sigma_{f} + (\mathcal{L}_{\tilde{g}} h(x)\pm \|\nabla h(\boldsymbol{x})\|\beta_{g} \sigma_{g} )\boldsymbol{u} + \alpha (h(x)) \geq 0\} 
% \end{aligned}
% \end{equation}

% \marginJC{I think there is a mathematical environment (assumptions+result+proof) that we can formulate here.}

\subsection{Control Synthesis via Second-Order Cone Programming}
\label{sec:socp_frame}

%\subsection{CLF-CBF-QP Framework}

We encode the control objective using the notion of a Lyapunov function. Formally, we assume the existence of a control Lyapunov function, as defined next.

\begin{definition}[{\cite{cbf}}]
A \emph{control Lyapunov function (CLF)} for the system dynamics in \eqref{eq: dynamic} is a continuously differentiable function $V: \mathbb{R}^n \mapsto {\mathbb{R}_{\geq 0}}$ for which there exist a class $\mathcal{K}$ function $\alpha_V$ such that, for all $\boldsymbol{x} \in \mathcal{X}$:
% \NA{Why do we commit to a constant $c_3$ instead of a class $\calK$ function?}
\begin{align*}
    % c_1\|\boldsymbol{x}\|^2 \leq &V(\boldsymbol{x}) \leq c_2\|\boldsymbol{x}\|^2\\
    \inf_{\boldsymbol{u}\in \mathcal{U}}[ \mathcal{L}_f V(\boldsymbol{x}) + \mathcal{L}_g V(\boldsymbol{x})\boldsymbol{u} + \alpha_V( V(\boldsymbol{x}))  ] \leq 0 .
\end{align*}
\end{definition}

The function $V$ may be used to encode a variety of control objectives, including for instance path following. % For the points $r(i)$, $i \in \mathcal{I}$, on the reference path~$r$, the Lyapunov function is used to ensure the convergence of the position in the robot state $\boldsymbol{x}$ to the reference points $r(i)$.
We present a specific Lyapunov function for this purpose in Sec.\ref{evaluation}.

When the barrier function is precisely known, one can combine CLF and CBF constraints to synthesize a safe controller via the following QP:
% \NA{I suggest changing the cost to $\|L(\bfu - k(\bfx))\|^2 + \lambda \delta^2$.}
\begin{equation}
\label{eq:QP_origin}
\begin{aligned}
& \min_{\boldsymbol{u} \in \calU,\delta \in \bbR}\,\, \|L(\boldsymbol{x})^\top(\boldsymbol{u} - \tilde{\boldsymbol{u}}(\boldsymbol{x}))\|^2 + \lambda \delta^2\\
% \boldsymbol{u}^T H \boldsymbol{u} + p_d \delta^2 \\
\mathrm{s.t.} \, \,  &\mathcal{L}_f V(\boldsymbol{x}) + \mathcal{L}_g V(\boldsymbol{x})  \boldsymbol{u} + \alpha_V(V(\boldsymbol{x})) \leq \delta \\
&\mathcal{L}_{f} h(\boldsymbol{x}) + \mathcal{L}_{g} h(\boldsymbol{x})\boldsymbol{u} + \alpha_h(h(\boldsymbol{x})) \geq 0,
\end{aligned}
\end{equation}
% %
% \marginJC{We are dealing with vectors in $\underbar{$\boldsymbol{u}$} \leq \boldsymbol{u} \leq \bar{\boldsymbol{u}}$, so we need to explain what ineqs mean (component-wise??)}
% %
where $\tilde{\boldsymbol{u}}(\boldsymbol{x})$ is a baseline controller, $L(\bfx)$ is a matrix penalizing control effort, and $\delta \geq 0$ (with the corresponding penalty $\lambda$) is a slack variable, introduced to relax the CLF constraints in order to ensure the feasibility of the QP. The baseline controller $\tilde{\boldsymbol{u}}(\boldsymbol{x})$ is used to specify additional control requirements such as desirable velocity or orientation (see Sec.~\ref{evaluation}) but may be set to $\tilde{\boldsymbol{u}}(\boldsymbol{x}) \equiv 0$ if minimum control effort is the main objective. The QP formulation in \eqref{eq:QP_origin} modifies $\tilde{\boldsymbol{u}}(\boldsymbol{x})$ online to ensure safety and stability via the CBF and CLF constraints.

Without exact knowledge of the barrier function $h$, we need to replace the CLF constraint in \eqref{eq:QP_origin} by \eqref{eq: new_set_h}:
%\eqref{eq:cbf} by \eqref{eq: new_set_h} in \eqref{eq:QP_origin} in order to synthesize a safe controller:
%
\begin{equation}
\label{eq:QP_tobeconvert}
\begin{aligned}
& \min_{\boldsymbol{u} \in \calU,\delta \in \mathbb{R}}\,\,  \|L(\boldsymbol{x})^\top(\boldsymbol{u} - \tilde{\boldsymbol{u}}(\boldsymbol{x}))\|^2 + \lambda \delta^2 \\
\text{s.t.} \, \, &\mathcal{L}_f V(\boldsymbol{x}) + \mathcal{L}_g V(\boldsymbol{x})  \boldsymbol{u} + \alpha_V(V(\boldsymbol{x})) \leq \delta \\
&\mathcal{L}_{f} \tilde{h}(\boldsymbol{x}) + \mathcal{L}_{g} \tilde{h}(\boldsymbol{x})\boldsymbol{u}+ \alpha_h(\tilde{h}(\boldsymbol{x}) - e_h(\boldsymbol{x}))  \\
&\phantom{\mathcal{L}_{f} \tilde{h}(\boldsymbol{x}) + \mathcal{L}_{g} \tilde{h}(\boldsymbol{x})\boldsymbol{u} + } \geq \|f(\boldsymbol{x}) + g(\boldsymbol{x})\boldsymbol{u}\| \,e_{\nabla h}(\boldsymbol{x}).
\end{aligned}
\end{equation}
% %
% \marginJC{Notice I have removed $u_{ref}$, see margin below, so the statements later need to be made consistent.}
% %
This makes the optimization problem in \eqref{eq:QP_tobeconvert} no longer a quadratic program. However, the following result shows that~\eqref{eq:QP_tobeconvert} is a (convex) second-order cone program (SOCP).
%, and can therefore be solved efficiently.

\begin{proposition}
\label{pro:socp}
% Let $\hat{\boldsymbol{u}} = (\boldsymbol{u}, \delta)$ be an augmented vector and $\hat{\boldsymbol{u}}_{ref} = (\boldsymbol{u}_{ref},0)$ is the augmented reference control input.
% Define $\widehat{L_g V(\boldsymbol{x})} = [L_g V(\boldsymbol{x}), -1]$, $\widehat{\mathcal{L}_{g} \tilde{h}(\boldsymbol{x})} = [\mathcal{L}_{g} \tilde{h}(\boldsymbol{x}), 0]$, $\widehat{g(\boldsymbol{x})} = [g(\boldsymbol{x}), \boldsymbol{0}_{n \times 1}]$
% \NA{I don't like introducing the widehat notation. I think keeping $\bfu$ and $\delta$ as separate variables is completely fine since $\delta$ always appears linearly.}. 
%For $\boldsymbol{u} \in \mathcal{U}$, the optimization problem~\eqref{eq:QP_tobeconvert} is equivalent to the following SOCP: 
% A CLF-CBF based QP problem~\eqref{eq:QP} with an SOC constraint is an SOCP problem:
% \NA{You should not mention QP at all. Instead directly state something along the line of: ``The robust safety constraint in (7) in a second-order cone constraint in $\bfu$ and leads to a SOCP formulation for safe control:" and give the SOCP in (10) here. You can explain the connection the QP in the proof.}
The optimization problem in \eqref{eq:QP_tobeconvert} is equivalent to the following second-order cone program: %
\begin{equation}
\label{eq:SOCP_formulation}
\begin{aligned}
    & \min_{\boldsymbol{u} \in \calU, \delta\in \mathbb{R},l\in \mathbb{R} } \, \, l \\
    \mathrm{s.t.} \, \, &\mathcal{L}_f V(\boldsymbol{x}) + \mathcal{L}_g V(\boldsymbol{x})  \boldsymbol{u} + \alpha_V(V(\boldsymbol{x})) \leq \delta \\
    &\|f(\boldsymbol{x}) + g(\boldsymbol{x})\boldsymbol{u}\| \,e_{\nabla h}(\boldsymbol{x}) \leq  \mathcal{L}_{f} \tilde{h}(\boldsymbol{x}) + \mathcal{L}_{g} \tilde{h}(\boldsymbol{x})\boldsymbol{u}\\ &+ \alpha_h(\tilde{h}(\boldsymbol{x}) - e_h(\boldsymbol{x})) \\
    &\Bigg\| \begin{bmatrix} 2L(\boldsymbol{x})^\top(\boldsymbol{u}-\tilde{\boldsymbol{u}}(\boldsymbol{x}))\\ 2\sqrt{\lambda} \delta \\ l-1 \end{bmatrix} \Bigg\| \leq l+1.
    % (L^T (\hat{\boldsymbol{u}}-\hat{\boldsymbol{u}}_{ref}), l, 1) \\
    % &\in \mathcal{Q}^n_{rot}\\
\end{aligned}
\end{equation}
% %
% \marginJC{I don't understand why we have both a Lyap function $V$ and a reference input $u_{ref}$. It's either one or the other, but not both, no?!}
% %
\end{proposition}
% % %
% \marginJC{$\alpha_V$ used here, but never introduced (see earlier margin)}
% % %

% %
% \marginJC{This is the optimization problem that should be in the statement of the result, no? As currently stated, the proposition is stating an obvious thing: that the formulated SOCP problem is SOCP. What we want to state instead is that the QP problem (9) is in fact an SOCP. Statements should be as simple/clean as possible, proofs should have the complexity (as much as needed, but no more than needed) }
% %
% Note that the convex QP problem can be converted into a SOCP problem, therefore QP can be viewed as a special case of SOCP. 
% Recall that the standard form a SOCP model is:
% \begin{equation}
% \label{eq:socp}
% \begin{aligned}
%     &\min_{\boldsymbol{x}} \, c^T \boldsymbol{x}\\
%     \text{s.t.} \, \, & \|C_i \boldsymbol{x} + d_i \| \leq e_i^T \boldsymbol{x} + f_i\\
%     & A_i \boldsymbol{x} = b_i
% \end{aligned}
% \end{equation}
% for $i = 1 \dots m$ and $\boldsymbol{x} \in \mathbb{R}^n$. 
% We show that the problem~\eqref{eq:QP_tobeconvert} is a valid SOCP problem. 

\begin{proof}
We first introduce a new variable $l$ so that the problem in \eqref{eq:QP_tobeconvert} is equivalent to
\begin{equation}
\label{eq:SOCP}
\begin{aligned}
    & \min_{\boldsymbol{u} \in \calU, \delta\in \mathbb{R},l\in \mathbb{R} } \, \, l \\
    \text{s.t.} \, \, &\mathcal{L}_f V(\boldsymbol{x}) + \mathcal{L}_g V(\boldsymbol{x}) \boldsymbol{u} + \alpha_V(V(\boldsymbol{x})) \leq \delta \\
    &\|f(\boldsymbol{x}) + g(\boldsymbol{x}) \boldsymbol{u}\| \,e_{\nabla h}(\boldsymbol{x}) \leq  \mathcal{L}_{f} \tilde{h}(\boldsymbol{x}) + \mathcal{L}_{g} \tilde{h}(\boldsymbol{x})\boldsymbol{u}\\ &+ \alpha_h(\tilde{h}(\boldsymbol{x}) - e_h(\boldsymbol{x})) \\
    &\|L(\boldsymbol{x})^\top(\boldsymbol{u}-\tilde{\boldsymbol{u}}(\boldsymbol{x}))\|^2 + \lambda \delta^2 \leq l.
\end{aligned}
\end{equation}
% where $\hat{H}$ is the corresponding augmented matrix of $H$ in~\eqref{eq:QP_origin} with the relaxation parameter penalty on its diagonal.
%
The last constraint in \eqref{eq:SOCP} corresponds to a rotated second-order cone, $\mathcal{Q}^n_{rot} \coloneqq \{(\boldsymbol{x}_{r},y_{r},z_{r}) \in \mathbb{R}^{n+2}\, | \,\|\boldsymbol{x}_{r}\|^2 \leq y_{r}z_{r}, y_{r} \geq 0, z_{r} \geq 0 \}$, which can be converted into a standard SOC constraint~\cite{alizadeh2003second}:
\[
\left\| \begin{bmatrix} 2 \boldsymbol{x}_{r} \\ y_r - z_r\end{bmatrix} \right\| \leq y_r+z_r.
\]
%We utilize the rotated second-order cone, $\mathcal{Q}^n_{rot} \coloneqq \{(\boldsymbol{x}_{r},y_{r},z_{r}) \in \mathbb{R}^{n+2}\, | \,\|\boldsymbol{x}_{r}\|^2 \leq y_{r}z_{r}, y_{r} \geq 0, z_{r} \geq 0 \}$,
% \NA{It is not clear how the proof is related to $\mathcal{Q}^n_{rot}$ or why we define $y_r$ and $z_r$.} 
%to convert the last constraint in~\eqref{eq:SOCP} into standard SOC constraint. 
Let $y_r = l$, $z_r = 1$ and consider the constraint $\|L(\boldsymbol{x})^\top(\boldsymbol{u}-\tilde{\boldsymbol{u}}(\boldsymbol{x}))\|^2 + \lambda \delta^2 \leq l$.
%\CQ{Now this form is not exactly the same form as $\mathcal{Q}^n_{rot} is $, but we can use augmented notation to make it same. Should we do that or this is enough?}
%
%\marginJC{I think this is enough!}
%
Multiplying both sides by $4$ and adding $l^2+1$, makes the constraint equivalent to 
\[
4\| L(\boldsymbol{x})^\top (\boldsymbol{u}-\tilde{\boldsymbol{u}}(\boldsymbol{x})) \|^2 + 4\lambda\delta^2 + (l-1)^2 \leq (l+1)^2.
\]
Taking a square root on both sides, we end up with $\sqrt{ \|2L(\boldsymbol{x})^\top (\boldsymbol{u}-\tilde{\boldsymbol{u}}(\boldsymbol{x}))\|^2 + (2\sqrt{\lambda}\delta)^2 + (l-1)^2} \leq l+1$, which is equivalent to the third constraint in \eqref{eq:SOCP_formulation}.
\end{proof}

For multiple obstacles in the environment, one can  add multiple CBF constraints to \eqref{eq:SOCP_formulation}. We leave for future work the characterization of the Lipschitz continuity properties of the controller resulting from \eqref{eq:SOCP_formulation}.

\section{Evaluation}
\label{evaluation}

We use the PyBullet simulator \cite{coumans2020} to evaluate our approach for online shape estimation and safe navigation. We use a TurtleBot, equipped with a LiDAR scanner with a $270$ degree field of view, $150$ rays per scan, $3$ meter range, and zero-mean Gaussian measurement noise with standard deviation $\sigma = 0.01$. The simulation environments contain obstacles with various shapes, a priori unknown to the robot.
% The goal is to show that the robot can avoid obstacles with the estimated barrier functions while following the desired path whenever possible.

%The robot is equipped with a LiDAR with $150$ rays per scan. Zero-mean Gaussian noise with standard deviation $\sigma = 0.01$ was added to the distance measurements along each ray. 

\subsection{Modeling}
%The two-wheeled robot can be steered by controlling the angular velocities of the wheels. 
% Take the length of wheel axis as L, we can define the dynamics as follows:
% \begin{equation}\label{eq:wheel}
% \begin{bmatrix} 
% \dot{x} \\ \dot{y} \\ \dot{\theta}
% \end{bmatrix} 
% = 
% \begin{bmatrix}
% \frac{R}{2} (\omega_r + \omega_l) \cos(\theta)\\ \frac{R}{2} (\omega_r + \omega_l) \sin(\theta)\\  \frac{R}{L}(\omega_r - \omega_l)\\
% \end{bmatrix}
% \end{equation}
% in which the state $(x,y) \in \mathbb{R}^2$ represents the robot position, $\theta$ is the robot orientation.
% \subsubsection{Standard Unicycle Model}
% A standard unicycle model representing such kind of robots is defined as follows:
% \begin{equation}
% \left[ \begin{array}{c} \dot{x} \\ \dot{y} \\ \dot{\theta} \end{array} \right] = \begin{bmatrix} v\cos(\theta) \\ v\sin(\theta) \\ \omega \end{bmatrix}  \label{eq:model1}
% \end{equation}
% where $v, \omega$ represent the robot longitudinal velocity and angular velocity, respectively, and the state and input are $\boldsymbol{x} = [x, y,\theta]^T \in \mathbb{R}^2 \times \{-\pi,\pi \}, \boldsymbol{u} = [v,w]^T \in \mathbb{R}^2$. 

We model the robot motion using unicycle kinematics and take a small distance $a \neq 0$ off the wheel axis as in \cite{cortes2017coordinated} to obtain a relative-degree-one model:
\begin{equation}
\label{eq:model2}
\begin{bmatrix} \dot{x} \\ \dot{y} \\ \dot{\theta} \end{bmatrix} = \begin{bmatrix} \psi \cos(\theta) - a\omega \sin(\theta)\\ \psi \sin(\theta) + a\omega \cos(\theta) \\ \omega  \end{bmatrix},
% = f(\boldsymbol{x}) + g(\boldsymbol{x}) \, \boldsymbol{u} + d, \qquad \|d\| \leq B 
\end{equation}
where $\psi$, $\omega$ represent the robot linear and angular velocity, respectively. The state, input, and output are $\boldsymbol{x} := [x, y,\theta]^\top \in \mathbb{R}^2 \times [-\pi,\pi)$, $\boldsymbol{u} := [\psi,\omega]^\top \in \mathbb{R}^2$, and $\boldsymbol{y} = v(\boldsymbol{x}) := [x, y] \in \mathbb{R}^2$. We use a baseline controller $\tilde{\boldsymbol{u}}(\boldsymbol{x}) \equiv [\psi_{max}, 0]^\top$ to encourage the robot to drive at max velocity $\psi_{max}$ in a straight line whenever possible. The desired robot path is specified by a 2-D curve $r(\gamma)$, $\gamma \in [0,1]$. To capture the path-following task via a CLF constraint, we define $\boldsymbol{\eta}(\boldsymbol{x}) := v(\boldsymbol{x}) - r(\gamma(\boldsymbol{x}))$. The closest point on the reference path $r(\gamma)$ to the robot state $\boldsymbol{x}$ is obtained by $\gamma(\boldsymbol{x}) := \arctan(y/x) /(2\pi) \in [0,1]$. According to~\cite{ames2014rapidly}, the following is a valid CLF for  path following:
\begin{equation}
  \label{eq:clf} 
  V(\boldsymbol{x}) = [\boldsymbol{\eta}^\top(\boldsymbol{x}),\dot{\boldsymbol{\eta}}^\top(\boldsymbol{x})]P[\boldsymbol{\eta}^\top(\boldsymbol{x}),\dot{\boldsymbol{\eta}}^\top(\boldsymbol{x})]^\top, 
\end{equation}
where $P$ is a positive-definite matrix calculated by solving the Lyapunov equation of the input-output linearization.

\begin{figure}[t]
\centering
\subcaptionbox{Ball\label{fig:3a}}{\includegraphics[width=0.235\linewidth]{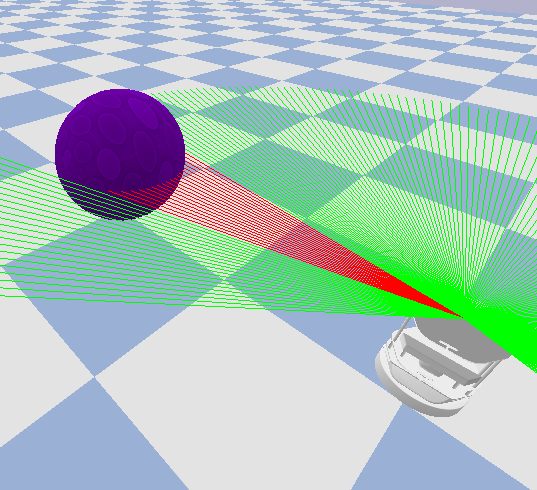}}%
\hfill
\subcaptionbox{Sofa\label{fig:3b}}{\includegraphics[width=0.24\linewidth]{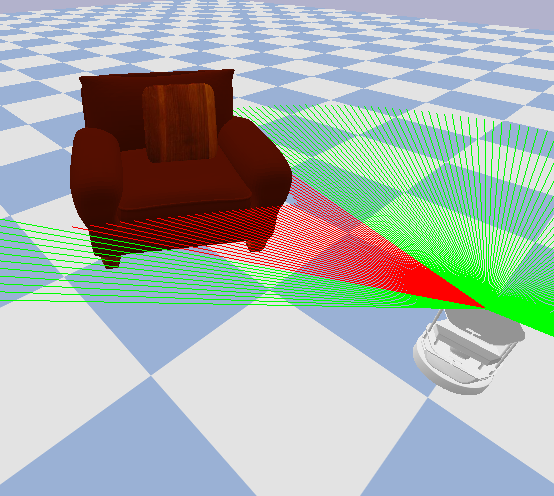}}
\hfill
\subcaptionbox{Table\label{fig:3c}}{\includegraphics[width=0.244\linewidth]{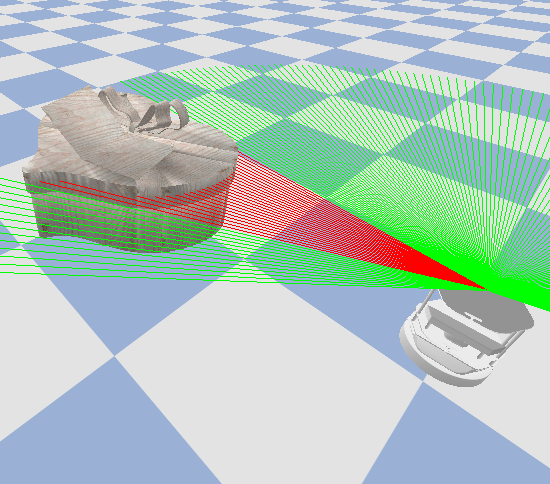}}
\hfill
\subcaptionbox{Toy\label{fig:3d}}{\includegraphics[width=0.24\linewidth]{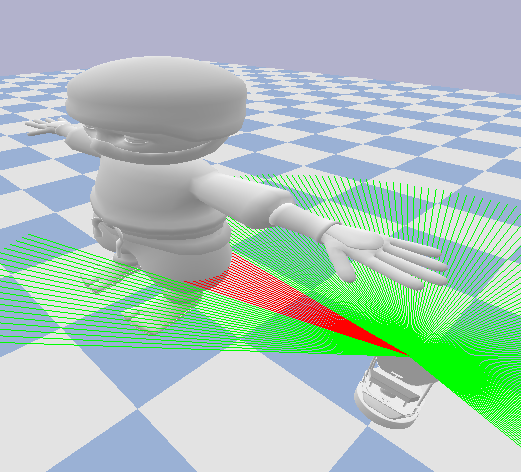}}
\\
\subcaptionbox{Dog\label{fig:3e}}{\includegraphics[width=0.24\linewidth]{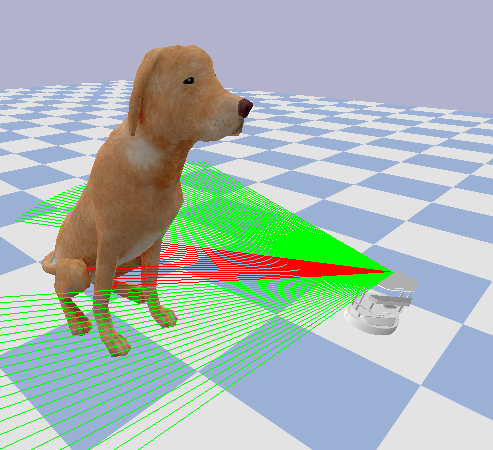}}%
\hfill
\subcaptionbox{Duck\label{fig:3f}}{\includegraphics[width=0.24\linewidth]{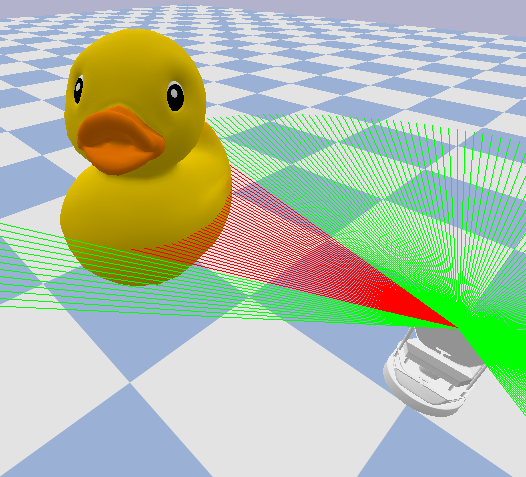}}%
\hfill
\subcaptionbox{Rabbit\label{fig:3g}}{\includegraphics[width=0.238\linewidth]{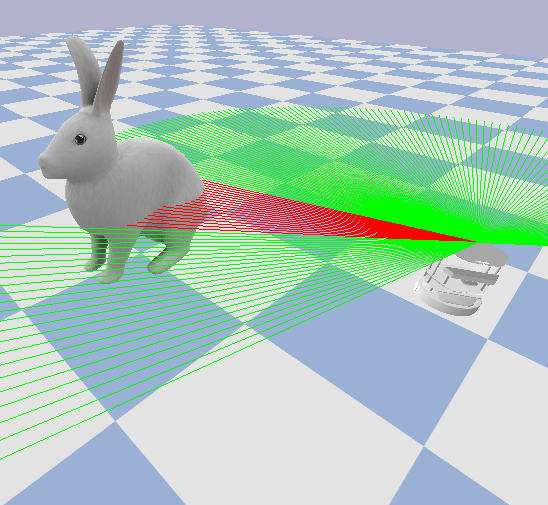}}%
\hfill
\subcaptionbox{Cat\label{fig:3h}}{\includegraphics[width=0.24\linewidth]{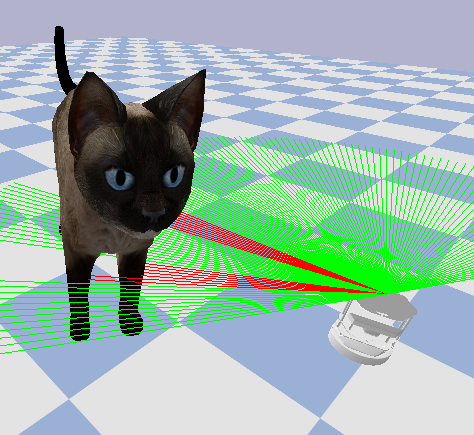}} \\
\caption{Object instances used to evaluate the estimation performance of online signed distance function approximation.}
\label{fig:sdf_result}
  \includegraphics[width=0.41\linewidth]{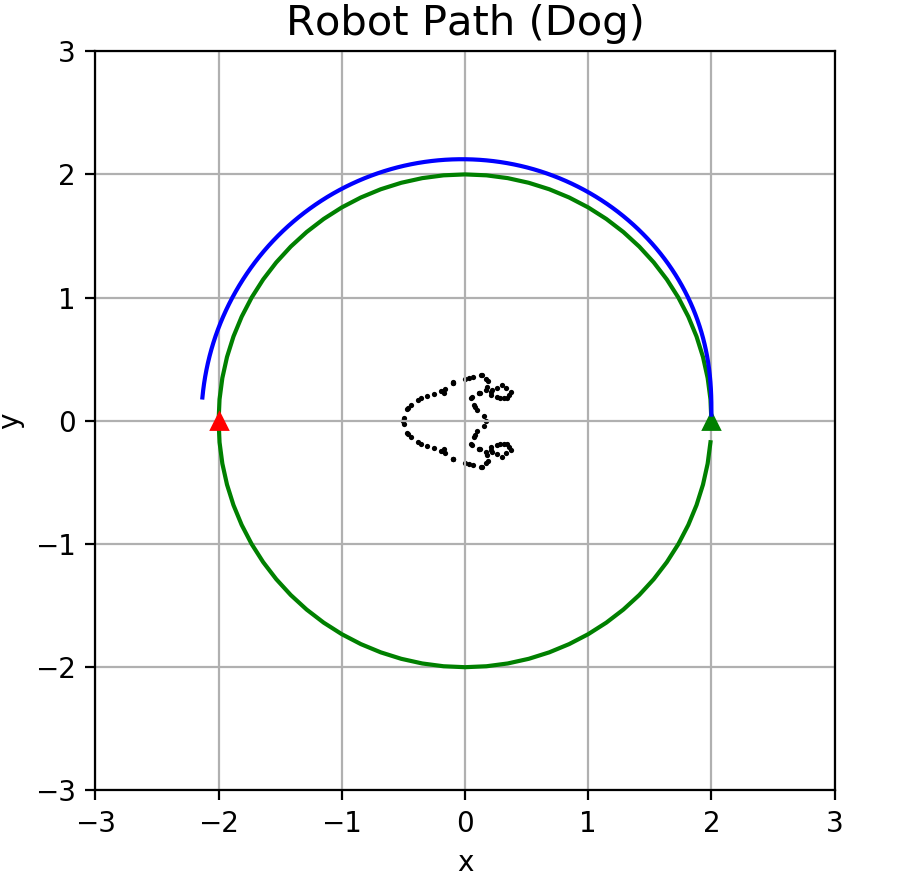}%
  \hfill%
  \includegraphics[width=0.57\linewidth]{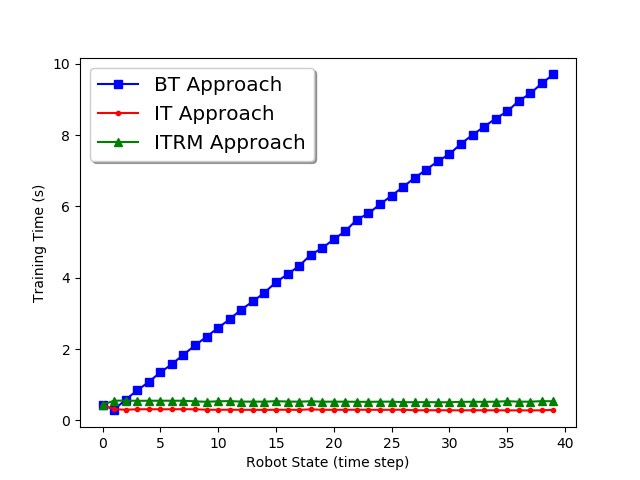}
  %\subcaptionbox{Robot Path\label{fig:4a}}{\includegraphics[width=0.48\linewidth]{fig/robot_path_dog.png}}
  %\hfill
  %\subcaptionbox{Training Time \label{fig:4b}}{\includegraphics[width=0.49\linewidth]{fig/training_time.png}}
  \caption{Robot motion (left) and training time (right) for the SDF estimation experiment in Sec.~\ref{sec:estimation_result}. The robot starts at $(2,0)$ and follows a circular reference path (green) to a goal position at $(-2,0)$ (red triangle). The actual trajectory followed by the robot is shown in blue. The black dots are points on the surface of a ground-truth Dog object and vary for different object instances. The average training times for the BT, IT, and ITRM methods for SDF approximation of the $8$ objects in Fig.~\ref{fig:sdf_result} are compared.}
  \label{fig:training_time}
\end{figure}

\begin{table}[t]
\centering
\caption{SDF estimation error \eqref{eq:valid_error} of the IT, ITRM, and BT training methods for LiDAR measurements with zero-mean Gaussian range noise with standard deviation $\sigma = 0.01$.}
\begin{tabular}{ |p{1.15cm}|p{2cm}|p{1.9cm}|p{1.9cm}| }
\hline
Case & $\mathcal{E}$ in IT & $\mathcal{E}$ in  ITRM  & $\mathcal{E}$ in  BT \\
\hline
1 (Ball) & $0.0287$  & $0.0148$ & $\mathbf{0.0140}$  \\
2 (Sofa)  & $0.0653$  & $\mathbf{0.0212}$ & $0.0215$ \\
3 (Table)  & $0.0564$ & $0.0230$ & $\mathbf{0.0227}$ \\
4 (Toy) & $0.1064$   & $0.0242$ & $\mathbf{0.0198}$\\
5 (Dog) & $0.0328$ & $0.0211$ & $\mathbf{0.0123}$\\
6 (Duck)  & $0.0751$  & $\mathbf{0.0102}$ & $0.0111$\\
7 (Rabbit)  & $0.0603$ & $0.0134$ & $\mathbf{0.0109}$\\
8 (Cat)  & $0.0292$ & $0.0149$ & $\mathbf{0.0123}$ \\
Average & $0.0568$ &$0.0179$ & $\mathbf{0.0147}$ \\
\hline
\end{tabular}
\label{table:LIDAR-w-noise}
\end{table}

\subsection{SDF Estimation Results}
\label{sec:estimation_result}
We compare the training time and prediction accuracy of the proposed ITRM approach for SDF estimation versus the IT and BT approaches described in Sec.~\ref{sec: incremental} for various 2-D obstacle shapes. Figs.~\ref{fig:sdf_result} and~\ref{fig:training_time} show our experiment setup. We used the multilayer perceptron architecture proposed by Gropp \textit{et al}. \cite{implicitsurface} and Park \textit{et al}. \cite{deepsdf} for ITRM training. The details of the neural network architecture are specified in Sec.~\ref{sec:network_archtecture}. We measure SDF approximation error as:
\begin{equation}
\label{eq:valid_error}
\mathcal{E} = \frac{1}{m}\sum_{i=1}^{m} |\tilde\varphi(\boldsymbol{y}_{i};\bftheta_k) |,   
\end{equation}
where $\{\boldsymbol{y}_i\}_{i=1}^{m}$ are $m = 500$ points uniformly sampled on the surface of the ground-truth object. The SDF error is shown in Table~\ref{table:LIDAR-w-noise} for the eight object instances in Fig.~\ref{fig:sdf_result}. The error of our ITRM approach is comparable with the error of the BT approach and is much smaller than the error of the IT approach. The training update time is the time needed for updating the network parameters from $\bftheta_{k-1}$ to $\bftheta_{k}$. The average training update time across the $8$ object instances for the three methods is shown in Fig.~\ref{fig:training_time}. We see that as the robot moves around the environment, the average training update time of the ITRM approach remains at about $0.7$s while the BT approach requires more and more time for training.

%of the SDF approximation $\tilde\varphi(\boldsymbol{y};\bftheta_k)$, we sample $m = 500$ points $\{\boldsymbol{y}_i\}_{i=1}^{m}$ %uniform on the surface of the ground-truth object and measure the SDF error ($k$ varies for different objects):
% \NA{What is the value of $k$, i.e., how many LiDAR scans were used in Table 1?}
%\begin{equation}
%    \mathcal{E} = \frac{1}{m}\sum_{i=1}^{m} |\tilde\varphi(\boldsymbol{y}_{i};\bftheta_k) |.
%    \label{eq:valid_error}
%\end{equation}
%

\begin{table}[t]
\centering
\caption{Training time and validation error of different neural network configurations for SDF approximation. The number of network layers, the number of training epochs, and whether a GPU is used are varied. The results are obtained using the settings in Table~\ref{table:LIDAR-w-noise} for the dog object shown in Fig.~\ref{fig:sdf_result}. The SDF training time includes both pre-processing of the LiDAR data and neural network training. We report the average time per LiDAR scan along the robot path. The SDF validation error is computed via \eqref{eq:valid_error}.}
{
%\begin{tabular}{ |p{2cm}|p{2.5cm}|p{2.3cm}| }
\begin{tabular}{ |l|c|c|c|c| }
\hline
\multirow{3}{2cm}{Training Setting (number of layer, training epochs)} & \multicolumn{2}{l|}{Average Training Time} & \multicolumn{2}{l|}{SDF Validation Error}\\\cline{2-5}
& \multirow{2}{*}{GPU} & \multirow{2}{*}{CPU} & \multirow{2}{*}{GPU} & \multirow{2}{*}{CPU}\\
& & & &\\
\hline
 4, 5  &  $ \bold{0.0744} $& $0.197 $ & $0.0212$ & $0.0266$\\
 4, 10 & $0.143$ & $0.388 $ & $0.0215$ & $0.0207$\\
 4, 16 &  $0.224$ & $0.716$ & $0.0207$ & $0.0218$\\
 6, 5  &  $ 0.0978$ & $0.407$ & $\bold{0.0132}$ & $0.0215$\\
 6, 10 &  $0.193$ &$0.731$ & $0.0136$&$0.0193$ \\
 6, 16 & $ 0.312$&$1.258$ & $0.0149$&$0.0158$ \\
 8, 5  & $0.178$&$0.544$ & $0.0181$&$0.0208$  \\
 8, 10 & $0.325$&$1.127$ & $0.0201$&$0.0184$  \\
 8, 16 & $0.497$&$1.841$ & $ 0.0189$&$0.0156$ \\
\hline
\end{tabular}
}
\label{table:training_setting}
\end{table}

\subsection{Network Architecture and Training Implementation}
\label{sec:network_archtecture}

To accelerate the training time of the ITRM approach further to support online navigation, we  explore the trade-off between training time and accuracy for different neural network configurations. The baseline neural network used in Sec.~\ref{sec:estimation_result} has $8$ fully connected layers with $512$ neurons each and a single skip connection from the input to the middle layer. All internal layers use Softplus nonlinear activations. We set the parameter $\lambda = 0.1$ in \eqref{eq:loss}. At each time step~$t_k$, we trained the network for $16$ epochs with the ADAM optimizer \cite{Adam} with constant learning rate of $0.001$. We also experimented with neural network configurations with fewer layers ($4$ or $6$) and different number of training epochs ($5$, $10$, or $16$), while keeping fixed the rest of the architecture. As real-time navigation necessitates obstacle shape estimation to be as quick as possible, we aim to obtain the smallest network architecture, trained with as few epochs as possible, that still maintains good SDF prediction accuracy. We report training time and accuracy results using a GPU (one Nvidia Geforce 2080 Super) or a CPU (Intel i7 9700K) in Table~\ref{table:training_setting}. The best trade-off between training time and SDF error for a 2-D shape is obtained using a GPU to train a $6$ layer network for $5$ epochs. This configuration enables training times of less than $0.1$s, suitable for real-time navigation.

\subsection{Safe Navigation Results}
\label{sec:navigation_results}

\begin{figure}[t]
\centering
\subcaptionbox{Environment 1\label{fig:5a}}{\includegraphics[width=0.47\linewidth]{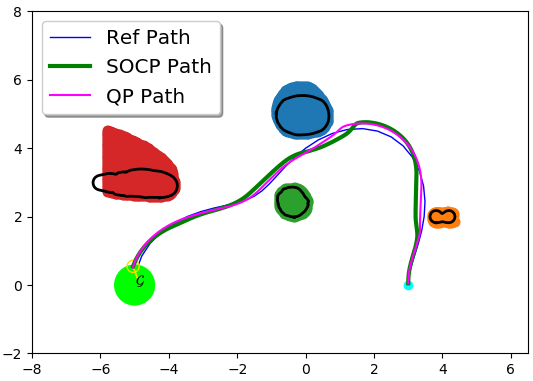}}%
\hfill%
\subcaptionbox{Environment 2\label{fig:5b}}{\includegraphics[width=0.47\linewidth]{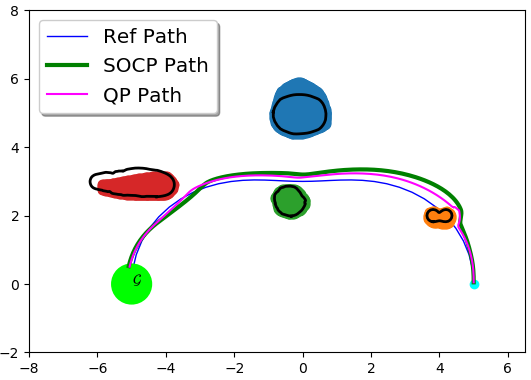}}\\
\subcaptionbox{Environment 6\label{fig:5c}}{\includegraphics[width=0.47\linewidth]{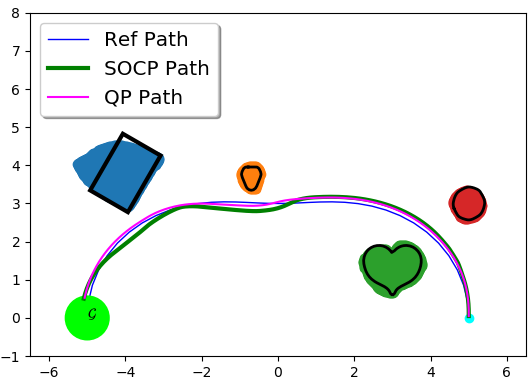}}%
\hfill%
\subcaptionbox{Environment 7\label{fig:5d}}{\includegraphics[width=0.47\linewidth]{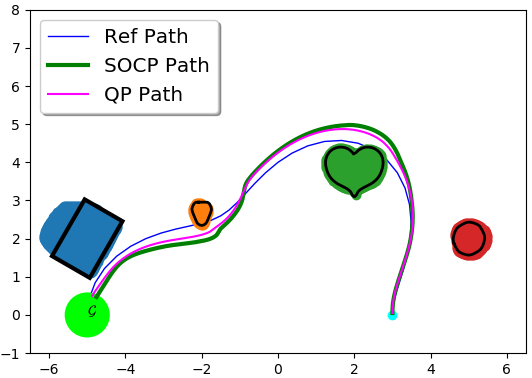}}\\
\caption{Simulation results for the CLF-CBF-SOCP and CLF-CBF-QP controllers. The reference path is shown in blue. The ground-truth obstacle surfaces are shown in black. The estimated obstacles, obtained after the whole path is traversed by the CLF-CBF-SOCP controller are shown in different colors (red, green, blue, orange). The trajectory generated by CLF-CBF-SOCP is shown in green, while the trajectory generated by CLF-CBF-QP is shown in pink. The starting point is cyan and the goal region is a light green circle.} 
\label{fig: result1}
\vspace*{-2ex}
\end{figure}

% In this section, we demonstrate that our proposed approach enables the robot to move around the environment safely by comparing the performance under online CLF-CBF-SOCP framework and online CLF-CBF-QP framework.  In each experimental environment such as Fig.~\ref{fig:1}, there are multiple obstacles which are unknown to the robot. The robot is asked to move along a path while avoiding obstacles. 

The second set of experiments demonstrates safe trajectory tracking using online SDF obstacle estimates to define constraints in the CLF-CBF-SOCP control synthesis optimization in \eqref{eq:SOCP_formulation}. The robot moves along a reference path while avoiding unknown obstacles in 8 different environments, similar to the one shown in Fig.~\ref{fig:1}. To account for the fact that the robot body is not a point mass, we subtract the robot radius $\tau = 0.177$ from the SDF estimate when defining the CBF: $\tilde{h}_i(\boldsymbol{x}) = \tilde{\varphi}_i(\boldsymbol{y};\bftheta) - \tau$. The error bounds of CBF $e_h(\boldsymbol{x})$ and its gradient $e_{\nabla h}(\boldsymbol{x})$ are approximated based on Table~\ref{table:LIDAR-w-noise}. We compare our CLF-CBF-SOCP approach to a CLF-CBF-QP. To account for estimation errors, it is possible to inflate the CBFs in the CLF-CBF-QP formulation by both the robot radius and the SDF approximation error, $\tilde{h}_i(\boldsymbol{x}) = \tilde{\varphi}_i(\boldsymbol{y};\bftheta) - \tau - e_h(\boldsymbol{x})$. This preserves linearity of the CBF constraints and leads to a more conservative controller. Comparing with \eqref{eq:SOCP_formulation}, we can see that the CLF-CBF-SOCP formulation accounts for both direct and gradient errors in the CBF approximations, and naturally reduces to a QP if $e_{\nabla h}(\boldsymbol{x})$ is zero. To emphasize the importance of accounting for estimation errors, we compare to a CLF-CBF-QP that assumes the estimated CBFs $\tilde{h}_i(\boldsymbol{x})$ are accurate and ignores the estimation errors.

%\textcolor{blue}{Defining instead the CBF as $\tilde{h}_i(\boldsymbol{x}) = \tilde{\varphi}_i(\boldsymbol{y};\bftheta) - \tau - e_h(\boldsymbol{x})$, in addition to being more conservative, would require  $e_h(\boldsymbol{x})$ to be continuously differentiable to compute the CLF-CBF-QP controller~\eqref{eq:QP_origin}.}

%\textcolor{blue}{If we simply define the CBF: $\tilde{h}_i(\boldsymbol{x}) = \tilde{\varphi}_i(\boldsymbol{y};\bftheta) - \tau - e_h(\boldsymbol{x})$ and perform the CLF-CBF-QP controller as \eqref{eq:QP_origin}, then we can achieve safer performance. However, our SOCP formulation further considers the gradient error $e_{\nabla h}(\boldsymbol{x})$ in \eqref{eq:SOCP_formulation}, which enforces safety more robustly.}

To avoid low velocities, we set diagonal of $L(\bfx)$ in Sec.~\ref{sec:socp_frame} as $l_1 = 10$, $l_2 = 1$, $l_3 = 10\sqrt{10}$, where $l_1$,$l_2$,$l_3$ are the penalty parameters for linear velocity, angular velocity and path following, respectively. If there is no solution found at some time step, $l_1$ is divided by $\sqrt{2}$ until a feasible solution found. We use the Fr\'echet distance between the reference path $r$ and the paths produced by the CLF-CBF-SOCP and CLF-CBF-QP controllers to evaluate their similarities. For paths $\displaystyle{A}$, $\displaystyle{B}$, the Fr\'echet distance is computed by: 
\begin{equation}
    F(A,B) = \inf_{\alpha, \beta} \max_{t \in [0,1]} \left\{ d(\displaystyle{A(\alpha(t))}, \displaystyle{B(\beta (t))}) \right\} ,
    \label{eq: fre_distance}
\end{equation}
where $\alpha,\beta: [0,1] \mapsto [0,1]$ are continuous, non-decreasing, surjections and $d$ is the Euclidean distance in our case.

%. Two reparameterizations $\alpha,\beta$ of $[0,1]$ are continuous, non-decreasing, surjections $\alpha,\beta: [0,1] \mapsto [0,1]$. The Fr\'echet distance is computed by:
% as the infimum over all reparameterizations $\alpha$ and $\beta$ of $[0,1]$ of the maximum over all $t \in [ 0 , 1 ]$ of the distance in $\displaystyle{S}$ between $\displaystyle{A(\alpha(t))}$ and $\displaystyle{B(\beta (t))}$. That is: 
%\begin{equation}
%    F(A,B) = \inf_{\alpha, \beta} \max_{t \in [0,1]} \left\{ d(\displaystyle{A(\alpha(t))}, \displaystyle{B(\beta (t))}) \right\} ,
%    \label{eq: fre_distance}
%\end{equation}
%where $d$ is the Euclidean distance in our case.

\begin{figure}[t]
\centering
\subcaptionbox{Environment 3\label{fig:6a}}{\includegraphics[width=0.47\linewidth]{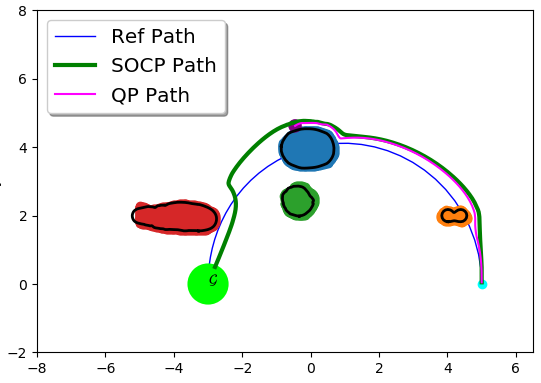}}%
\hfill
\subcaptionbox{Zoom out of (a)\label{fig:6b}}{\includegraphics[width=0.47\linewidth]{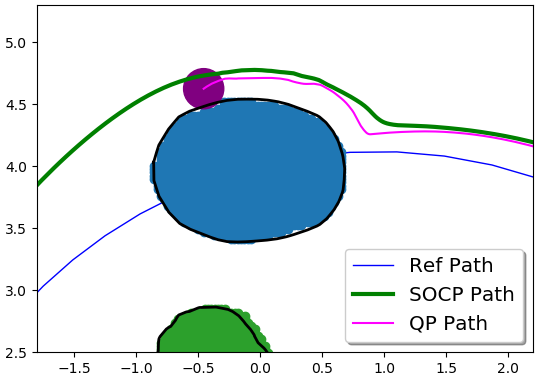}}\\
\subcaptionbox{Environment 8\label{fig:6c}}{\includegraphics[width=0.47\linewidth]{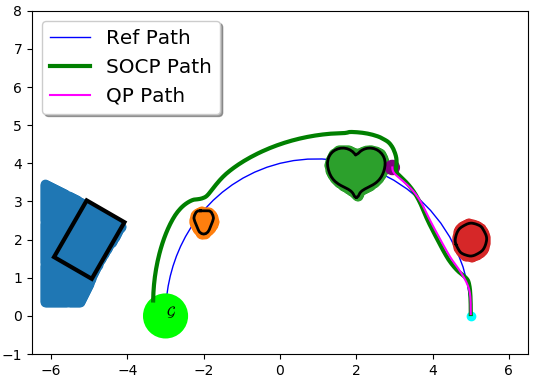}}%
\hfill
\subcaptionbox{Zoom out of (c)\label{fig:6d}}{\includegraphics[width=0.47\linewidth]{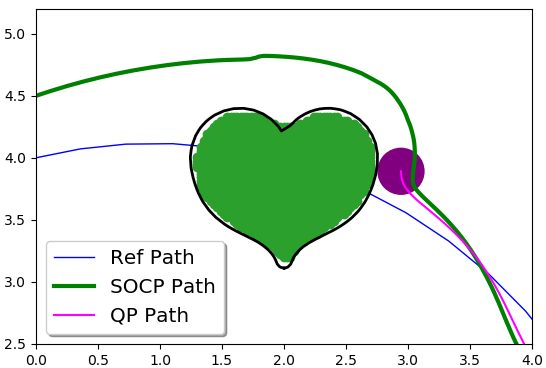}}
\caption{Simulation results in environments where the CLF-CBF-SOCP controller succeeds but CLF-CBF-QP one fails. The robot is shown as a purple circle when crashing into an obstacle using the CLF-CBF-QP controller. Fig.~\ref{fig:6b} and Fig.~\ref{fig:6d} show enlargements of the crash regions in Fig.~\ref{fig:6a} and Fig.~\ref{fig:6c}, respectively.}\label{fig: result_qp_failed}
\vspace*{-2ex}
\end{figure}

In Fig.~\ref{fig: result1}, the robot can follow the prescribed reference path if no obstacles are nearby.  The green trajectory generated by CLF-CBF-SOCP is always more conservative than the pink trajectory since it takes the errors in the CBF estimation into account. When there is an obstacle near or on the reference path, the robot controlled by the CLF-CBF-SOCP controller stays further away from the obstacles than the robot controlled by the CLF-CBF-QP controller. This agrees with the Fr\'echet distance results presented in Table~\ref{table:traj}. In Fig.~\ref{fig: result_qp_failed}, we see that the CLF-CBF-QP controller sometimes fails to avoid obstacles because it does not consider the errors in the CBF estimation. In contrast, the CLF-CBF-SOCP controller is guaranteed by Prop.~\ref{pro:socp} to remain safe if the CBF estimation is captured correctly in the SOC constraints.

%Fig.~\ref{fig:6b} and Fig.~\ref{fig:6d} show the robot body (purple circle) intersecting with the black obstacle outlines under the CLF-CBF-QP controller.

In summary, Table \ref{table:traj} indicates that the trajectory mismatch with respect to the reference path is larger under our CLF-CBF-SOCP controller than under the CLF-CBF-QP controller if they both succeed. However, our approach guarantees safe navigation, while the CLF-CBF-QP controller sometimes causes collisions due to errors in CBF estimation.

\begin{table}[t]
\centering
\caption{Fr\'echet distance between the reference path and the robot trajectories generated by the CLF-CBF-SOCP and the CLF-CBF-QP controllers (smaller values indicate larger trajectory similarity, N/A indicates that the robot failed to reach the goal region).}
\begin{tabular}{ |p{1.4cm}|p{2.5cm}|p{2.5cm}| }
\hline
Environment & CLF-CBF-SOCP & CLF-CBF-QP\\\hline
1 & $\mathbf{0.3428}$ & $0.3442$\\
2  & $0.5339$ & $\mathbf{0.3971}$ \\ 
3   & $0.8356$ & N/A \\
4 & $0.3968$ & $\mathbf{0.3492}$ \\
5  & $1.0535$ & N/A \\
6   & $0.3611$ & $\mathbf{0.3444}$\\
7 & $0.4643$ & $\mathbf{0.3493}$\\
8  &  $0.8674$ & N/A  \\\hline 
\end{tabular}
\label{table:traj}
\vspace*{-1ex}
\end{table}

% \begin{figure}[t]
%   \centering
%   \includegraphics[width=\linewidth]{fig/robot_path_w_est.png}
%   \caption{Robot Path with Estimation \TODO{See comments in the caption below. They apply here too. What are the different curves?}}
%   \label{fig: robot_path_estimation}
% \end{figure}

% \begin{figure}[t]
%   \centering
%   \includegraphics[width=\linewidth]{fig/robot_path_toy.png}
%   \caption{Robot Path with Estimation \TODO{Write captions like regular text rather than like titles. The captions should not be centered, should end with a period and should explain well what is going on in the figure. In this figure, why does the green trajectory enter the estimated blue region? Isn't this a hard constraint in the SOCP?}}
%   \label{fig: robot_path_estimation}
% \end{figure}

% \marginJC{The two "robot path with estimation" figures have the same label. Captions are missing.}

% \begin{figure*}
%   \subfloat[]{
% 	\begin{minipage}[c][0.7\width]{
% 	   0.45\textwidth}
% 	   \centering
% 	   \includegraphics[width=0.7\textwidth]{fig/SOCP_good2.png}
% 	\end{minipage}}
%  \hfill	
%   \subfloat[]{
% 	\begin{minipage}[c][0.7\width]{
% 	   0.46\textwidth}
% 	   \centering
% 	   \includegraphics[width=0.7\textwidth]{fig/QP_bad2.png}
% 	\end{minipage}}
% \caption{Simulation results using (a) CLF-CBF-SOCP framework and (b) CLF-CBF-QP framework with an aggressive initial position.}
% \label{fig: result2}
% \end{figure*}

\section{Conclusion}
% \& Future Work

We introduced an incremental training approach with replay memory to enable online estimation of signed distance functions and construction of corresponding safety constraints in the form of control barrier functions. The use of replay memory balances training time with estimation error, making our approach suitable for real-time estimation. We showed that accounting for the direct and gradient errors in the CBF approximations leads to a new CLF-CBF-SOCP formulation for safe control synthesis. Future work will consider capturing localization and robot dynamics errors in the safe control formulation and will apply our techniques to real autonomous navigation experiments.

\bibliographystyle{ieeetr}
\bibliography{./bib/ref.bib}

\end{document}